\documentclass[10pt, conference, letterpaper]{IEEEtran}
\IEEEoverridecommandlockouts
\usepackage{cite}
\usepackage{amsmath,amssymb,amsfonts}
\usepackage{amsthm}
\usepackage{algorithmic}
\usepackage{textcomp}
\usepackage{xcolor}
\usepackage[nocomma]{optidef}
\usepackage{bm}
\usepackage{tabularx}
\usepackage{algorithm}
\usepackage{algorithmic}
\usepackage{subfigure}
\usepackage[pdftex]{graphicx}
\usepackage{caption}
\usepackage{comment}
\usepackage{bbm}

\newcommand{\high}[1]{{\color{red}{#1}}}

\graphicspath{{figures/}} 
\newtheorem{theorem}{Theorem}
\newtheorem{lemma}{Lemma}
\DeclareMathOperator*{\argmax}{argmax}

\def\BibTeX{{\rm B\kern-.05em{\sc i\kern-.025em b}\kern-.08em
    T\kern-.1667em\lower.7ex\hbox{E}\kern-.125emX}}

\def\N{\mathcal{N}}
\def\PN{\mathcal{P}(\mathcal{N})}
\def\S{S}
\def\A{Z}
\begin{document}

\title{Combinatorial Sleeping Bandits with Fairness Constraints}

\author{
Fengjiao Li,
Jia Liu,
and Bo Ji 
\thanks{
This work was supported in part by the NSF under Grants CNS-1651947, CCF-1657162, ECCS-1818791, CCF-1758736, CNS-1758757, and CNS-1446582, the ONR under Grant N00014-17-1-2417, and the AFRL under Grant FA8750-18-1-0107. 
A preliminary version of this work without detailed proofs has been presented at IEEE INFOCOM 2019 \cite{li2019infocom}.

Fengjiao Li (fengjiao.li@temple.edu) and Bo Ji (boji@temple.edu) are with the Department of Computer and
Information Sciences, Temple University, Philadelphia, PA, USA. Jia Liu (jialiu@iastate.edu) is with the Department of Computer Science,
Iowa State University, Ames, IA, USA. Bo Ji is the corresponding author.}
}

\maketitle

\begin{abstract}
	The multi-armed bandit (MAB) model has been widely adopted for studying many practical optimization problems (network resource allocation, ad placement, crowdsourcing, etc.) with unknown parameters. The goal of the player (i.e., the decision maker) here is to maximize the cumulative reward in the face of uncertainty. However, the basic MAB model neglects several important factors of the system in many real-world applications, where multiple arms (i.e., actions) can be simultaneously played and an arm could sometimes be ``sleeping" (i.e., unavailable). Besides reward maximization, ensuring \emph{fairness} is also a key design concern in practice. To that end, we propose a new \emph{Combinatorial Sleeping MAB model with Fairness constraints}, called \emph{CSMAB-F}, aiming to address the aforementioned crucial modeling issues. The objective is now to maximize the reward while satisfying the fairness requirement of a minimum selection fraction for each individual arm. To tackle this new problem, we extend an online learning algorithm, called \emph{Upper Confidence Bound (UCB)}, to deal with a critical tradeoff between \emph{exploitation} and \emph{exploration} and employ the virtual queue technique to properly handle the fairness constraints. By carefully integrating these two techniques, we develop a new algorithm, called \emph{Learning with Fairness Guarantee (LFG)}, for the CSMAB-F problem. Further, we rigorously prove that not only LFG is \emph{feasibility-optimal}, but it also has a time-average \emph{regret} upper bounded by $\frac{N}{2 \eta} + \frac{\beta_1 \sqrt{m N T \log{T}}+ \beta_2 N}{T}$, where $N$ is the total number of arms, $m$ is the maximum number of arms that can be simultaneously played, $T$ is the time horizon, $\beta_1$ and $\beta_2$ are constants, and $\eta$ is a design parameter that we can tune. Finally, we perform extensive simulations to corroborate the effectiveness of the proposed algorithm. Interestingly, the simulation results reveal an important tradeoff between the regret and the speed of convergence to a point satisfying the fairness constraints. 

\end{abstract}


\section{Introduction}
The \emph{multi-armed bandit (MAB)} model has been widely adopted for studying many practical optimization problems (network resource allocation, ad placement, crowdsourcing, etc.) with unknown parameters (see, e.g., \cite{bubeck2012regret}). In the basic stochastic MAB setting, there are $N$ arms (i.e., actions), each of which, if played, returns a random reward to the player (i.e., the decision maker). The random reward of each arm takes values in $[0,1]$ and is assumed to be \emph{independent and identically distributed (i.i.d.)} over time. However, the reward distributions and the mean rewards are unknown \emph{a priori}. The player decides which single arm to play in each round for a given time horizon of $T$ rounds, with a goal of maximizing the cumulative reward in the face of unknown mean rewards.

However, this basic MAB model neglects several important factors of the system in many real-world applications, where multiple actions can be simultaneously taken and an action could sometimes be ``sleeping" (i.e., unavailable). 
Take wireless scheduling for example: multiple clients compete for a shared wireless channel to transmit packets to a common access point (AP). The AP decides which client(s) can transmit at what times. A successfully delivered packet will generate a random reward, which could represent the value of the information contained in the packet. In each scheduling cycle, multiple clients could be scheduled for simultaneous transmissions as the channel can typically be divided into multiple ``sub-channels" using multiplexing technologies \cite{rappaport2001}. On the other hand, some clients may be unable to transmit packets when experiencing a poor channel condition (due to fading or mobility). Furthermore, in addition to maximizing the reward, ensuring \emph{fairness} among the clients or providing \emph{Quality of Service (QoS)} guarantees to the clients is also a key design concern in wireless scheduling \cite{liu2003framework,hou2009theory}, as well as in network resource allocation in general \cite{5461911}. These important factors (i.e., combinatorial actions, availability of actions, and fairness) are commonly shared by many other applications too (see more detailed discussions in Section~\ref{sec:applications}). However, it remains largely unexplored in the literature to carefully integrate all these factors into a unified MAB model.

To that end, in this paper we propose a new \emph{Combinatorial Sleeping MAB model with Fairness constraints}, called \emph{CSMAB-F}, aiming to address the aforementioned modeling issues, which are practically important for a wide variety of applications. Compared to the basic MAB setting, in the proposed framework the set of available arms follows a certain distribution that is assumed to be \emph{i.i.d.} over time and is unknown \emph{a priori}. However, the information of available arms will be revealed at the beginning of each round. The player can then play multiple, but no more than $m$, available arms and receives a compound reward being the weighted sum of the rewards of the played arms. We also impose fairness constraints that the player must ensure a (possibly different) minimum selection fraction for each individual arm. The goal is now to maximize the reward while satisfying the fairness requirement. 
We summarize our main contributions as follows.

First, to the best of our knowledge, \emph{this is the first work that integrates all three critical factors of combinatorial arms, availability of arms, and fairness into a unified MAB model}. The proposed CSMAB-F framework successfully addresses these crucial modeling issues. This new problem, however, becomes much more challenging. In particular, integrating fairness constraints adds a new layer of difficulty to the combinatorial sleeping MAB problem that is already quite challenging. This is because not only the player encounters a \emph{fundamental tradeoff} between \emph{exploitation} (i.e., staying with the currently-known best option) and \emph{exploration} (i.e., seeking better options) when attempting to maximize the reward, but she is also faced with a \emph{new dilemma}: how to manage the balance between maximizing the reward and satisfying the fairness requirement? Several well-known MAB algorithms can successfully handle the exploitation-exploration tradeoff, but none of them was born with fairness constraints in mind.

To address this new challenge, we extend an online learning algorithm, called \emph{Upper Confidence Bound (UCB)}, to deal with the exploitation-exploration tradeoff and employ the \emph{virtual queue technique} to properly handle the fairness constraints. By carefully integrating these two techniques, we develop a new algorithm, called \emph{Learning with Fairness Guarantee (LFG)}, for the CSMAB-F problem. Further, we rigorously prove that not only LFG is \emph{feasibility-optimal}, but it also has a time-average \emph{regret} (i.e., the reward difference between an optimal algorithm that has \emph{a priori} knowledge of the mean rewards and the considered algorithm) upper bounded by $\frac{N}{2 \eta} + \frac{\beta_1 \sqrt{m N T \log{T}}+ \beta_2 N}{T}$, where $\beta_1$ and $\beta_2$ are constants and $\eta$ is a design parameter that we can tune. Note that our regret analysis is more challenging as the traditional regret analysis becomes non-applicable here due to the integration of virtual queues for handling the fairness constraints.

Finally, we conduct extensive simulations to elucidate the effectiveness of the proposed algorithm. From the simulation results, we observe that LFG can effectively meet the fairness requirement while achieving a good regret performance. Interestingly, the simulation results also reveal a critical tradeoff between the regret and the speed of convergence to a point satisfying the fairness constraints. We can control and optimize this tradeoff by tuning the value of parameter $\eta$. 


The rest of the paper is organized as follows. 
We first discuss related work and describe the proposed CSMAB-F framework in Sections~\ref{sec:relatedwork} and \ref{sec:system}, respectively. Then, we develop the LFG algorithm for the CSMAB-F problem in Section~\ref{sec:approach}, followed by the performance analysis in Section~\ref{sec:results}. Detailed discussions about several real-world applications are provided in Section~\ref{sec:applications}. Finally, we present simulation results in Section~\ref{sec:simulation} and make concluding remarks in Section~\ref{sec:conclusion}.    

\section{Related Work} \label{sec:relatedwork}
Starting with the seminal work of \cite{lai1985asymptotically}, the MAB problems have been extensively studied in a large body of work (see, e.g., \cite{gittins2011multi,bubeck2012regret}). 
In the basic MAB setting, the authors of \cite{lai1985asymptotically} establish a fundamental logarithmic lower bound on the regret of a class of ``uniformly good policies" and propose UCB policies that asymptotically achieve the lower bound. Further, the work of \cite{auer2002finite} shows that logarithmic regret can be achieved uniformly over time rather than asymptotically by simpler sample-mean-based UCB policies and an $\epsilon_t$-greedy policy. 

Following this line of research, different MAB variants have been proposed to model several important factors of the system in real-world applications. The ones that are relevant to ours include combinatorial MAB (CMAB) where multiple arms form a super arm and can be simultaneously played \cite{anantharam1987asymptotically,gai2012combinatorial,kvetonmatroid,chen2013combinatorial,chen2016combinatorial,combes2015combinatorial} and sleeping MAB (SMAB) where an arm could sometimes be ``sleeping" (i.e., unavailable)\cite{kleinberg2010regret,kanade2009sleeping,kanade2014learning,chatterjee2017analysis}.
Being the first to study the CMAB problem, the work of \cite{anantharam1987asymptotically} considers combinations of a fixed number of simultaneous plays. This simple combinatorial structure has been generalized to permutations \cite{gai2012combinatorial} and matroids \cite{kvetonmatroid}.
The work of \cite{chen2013combinatorial,chen2016combinatorial} generalizes linear reward functions considered in \cite{anantharam1987asymptotically,gai2012combinatorial,kvetonmatroid} to include a large class of linear and nonlinear rewards.
In \cite{combes2015combinatorial}, the authors prove a tight problem-specific lower bound for stochastic CMAB (where the reward of each played arm rather than the combinatorial reward is revealed) and propose an efficient sampling algorithm with an improved multiplicative factor.
The work of \cite{kleinberg2010regret} is among the first to study the SMAB problem. This work provides a computationally efficient algorithm for the setting of stochastic rewards while allowing both stochastic and adversarial availability. Follow-up work of \cite{kanade2009sleeping,kanade2014learning} studies the setting of adversarial rewards while the availability of arms is either stochastic or adversarial. Very recently, the authors of \cite{chatterjee2017analysis} analyze the performance of Thompson Sampling for the SMAB problem and show that it empirically performs better than other algorithms.
Another recent study in \cite{chen2018contextual} considers combinatorial sleeping MAB with submodular reward functions in the contextual bandit setting. This work develops a solution based on a well-known greedy algorithm for submodular maximization and prove that it can achieve a sublinear regret, which is in comparison to the greedy algorithm in the setting with known rewards.

%

MAB settings with constraints have also been considered in prior studies. Most of them focus on bandits with budgets (see, e.g., \cite{combes2015bandits}) or bandits with knapsacks (see, e.g., \cite{badanidiyuru2018bandits}), where no more plays can be made if the budget/knapsack constraints are violated. Hence, these types of constraints are very different from the \emph{long-term} fairness constraints we consider in this paper.
Some very recent work considers multi-type rewards \cite{denardo2013multi} and multi-level rewards \cite{cai2018learning,chen2018beyond}. They introduce a minimum guarantee requirement that the total reward of some type/level must be no smaller than a given threshold. However, these studies differ significantly from ours in the following key aspects. First, and most importantly, their constraints do not model fairness among arms. The required minimum guarantee is for the total rewards (of some type/level) rather than for each individual arm. Second, no learning algorithm is proposed in \cite{denardo2013multi}; the proposed learning algorithms in \cite{cai2018learning,chen2018beyond} may violate the constraints, although they show provable violation bounds. Third, they assume that all the arms are available at all times. Last but not least, the proof techniques for regret analysis in \cite{cai2018learning,chen2018beyond} are very different from ours.

Fairness in online learning has been studied in \cite{joseph2016fairness, joseph2016fair}. A key idea of their proposed fair algorithm is that two arms should be played with equal probability until they can be distinguished with a high confidence. Another work \cite{talebi2018learning} studies how to learn proportionally fair allocations by considering the maximization of a logarithmic utility function. These studies are less relevant to our work, although they share some high-level similarities with ours in modeling fairness.

At a technical level, the work of \cite{hsu2018integrate} that integrates learning and queueing is most related to ours. We follow a similar line of regret analysis in \cite{hsu2018integrate} for deriving the upper bound. However, they do not explicitly model fairness constraints, nor do they consider the availability of arms.

We notice that since the publication of our conference version \cite{li2019infocom},
the work of \cite{patil2019achieving} follows our model with a stronger fairness notion and proposes algorithms that can achieve an improved accumulative regret that is logarithmic. However, their proposed algorithms either are $T$-aware (i.e., assuming the knowledge of the length of the time horizon, $T$) or provide fairness guarantees for some special cases only, where the minimum selection fraction for every individual arm should be less than $1/k$. 

\section{System model and problem formulation} \label{sec:system}

In this section, we describe the detailed setting of our proposed CSMAB-F framework. 
Let $\N= \{1, 2, \dots, N\}$ denote the set of $N$ arms. Each arm $i \in \N$ is associated with a reward $X_i(t)$ in round $t$, where $t=0, 1,2,\dots$. The reward is a random variable on $[0,1]$ and follows a certain distribution with mean $\mu_i$. We assume that the reward for each arm is \emph{i.i.d.} over time. The mean reward vector $\bm{\mu}=(\mu_1,\dots,\mu_N)$ is unknown \emph{a priori}. 
In our setting, an arm could sometimes be ``sleeping" (i.e., unavailable).
Let $A(t) \in \PN$ denote the set of available arms in round $t$, where $\PN$ is the power set of $\N$.
We use $P_\mathbf{A}(\A) \triangleq P(A(t)=\A)$, where $\A \in \PN$, to denote the distribution of available arms, which is assumed to be \emph{i.i.d.} over time. This distribution is unknown \emph{a priori}, but the set of available arms $A(t)$ will be revealed to the player at the beginning of each round $t$. 

In each round, the player is allowed to play multiple, but no more than $m$, available arms (i.e., arms belonging to $A(t)$). 
Each subset of available arms is also called a \emph{super arm} \cite{chen2013combinatorial}. We restrict the size of a chosen super arm to be no larger than $m$ so as to account for resource constraints (see discussions on applications in Section \ref{sec:applications}). Let $\mathcal{S}(\A)$ represent the set of all feasible super arms when the set of available arms $\A$ is observed, i.e., $\mathcal{S}(\A) \triangleq \{\S \subseteq \A: |\S|\leq m \},$ where $|\S|$ denotes the cardinality of set $\S$. In round $t$, a player selects a super arm $S(t) \in \mathcal{S}(A(t))$ and receives a compound reward $R(t)$, which is a weighted sum of the rewards of the played arms, i.e., $R(t) \triangleq \sum_{i\in S(t)}w_iX_i(t)$, where $w_i$ is the weight of arm $i$. We assume that the weights $w_i$ are fixed positive numbers known \emph{a priori} and are upper bounded by a finite constant $w_{\max}>0$. The goal of the player is to maximize the expected time-average reward for a given time horizon of $T$ rounds, i.e., $\mathbb{E} [ \frac{1}{T} \sum_{t=0}^{T-1} R(t)]$.

To describe the action for each individual arm, we use a binary vector $\mathbf{d}(t) = (d_1(t), \dots, d_N(t))$ to indicate whether each arm is played or not in round $t$, where $d_i(t)=1$ if arm $i$ is played, i.e., $i \in S(t)$; otherwise, $d_i(t)=0$.
Then, the action vector $\mathbf{d}(t)$ must satisfy $\sum_{i=1}^N d_i(t) \leq m$ for all $t \ge 0$.

As we discussed in the introduction, in addition to maximize the reward, ensuring fairness among the arms is also a key design concern for many real-world applications. To model the fairness requirement, we introduce the following constraints on a minimum selection fraction for each individual arm:
\begin{equation}
\liminf_{T \to \infty} \frac{1}{T} \sum_{t=0}^{T-1}\mathbb{E}[d_i(t)]\geq r_i  \, \, \forall i \in \N,\label{eq:fraction requirement}
\end{equation}
where $r_i \in (0,1)$ is the required minimum fraction of rounds in which arm $i$ is played.
The minimum selection fraction vector $\mathbf{r}=(r_1, \dots, r_N)$ is said to be \emph{feasible} if there exists a policy that makes a sequence of decisions $S(t)$ for $t \ge 0$ such that \eqref{eq:fraction requirement} is satisfied.
Then, the \emph{maximal feasibility region} $\mathcal{C}$ is defined as the set of all such feasible vectors $\mathbf{r}\in (0,1)^N$. A policy is said to be \emph{feasibility-optimal} if it can support any vector $\mathbf{r}$ (i.e.,  \eqref{eq:fraction requirement} is satisfied) strictly inside the maximal feasibility region $\mathcal{C}$.

We now consider the special class of stationary and randomized policies called \emph{$A$-only policies}. An $A$-only policy observes the set of available arms $A(t)$ for each round $t$ and independently chooses a super arm $S(t) \in \mathcal{S}(A(t))$ as a (possibly randomized) function of the observed $A(t)$ only. An $A$-only policy $\alpha$ is characterized by a group of probability distributions, denoted by $\mathbf{q} = [q_{\S}(\A), \forall \S \in \mathcal{S}(\A), \forall \A \in \PN]$, where $q_{\S}(\A)$ is the probability that policy $\alpha$ chooses super arm $\S \in \mathcal{S}(\A)$ when observing the set of available arms $\A \in \PN$, and $\sum_{\S \in \mathcal{S}(\A)} q_{\S}(\A) = 1$ for all $\A \in \PN$.
Then, under policy $\alpha$, the action $d_i^{\alpha}(t)$ is \emph{i.i.d.} over time with the following mean:
\begin{equation}
\mathbb{E}[d_i^{\alpha}(t)] = \sum_{\A \in \PN } P_{\mathbf{A}}(\A) \sum_{\S \in \mathcal{S}(\A): i \in \S} q_{\S}(\A), \label{eq:optimal_action_expectation}
\end{equation}
for every arm $i \in \N$ and for all $t \ge 0$, and thus, constraint \eqref{eq:fraction requirement} is equivalent to $\mathbb{E}[d_i^{\alpha}(t)] \geq r_i$ for every arm $i \in \N$.
Further, we have the following lemma.

\begin{lemma}
If a vector $\mathbf{r}$ is strictly inside the maximal feasibility region $\mathcal{C}$, then there exists an $A$-only policy that can support vector $\mathbf{r}$. \label{lem:Lemma_randomized_existed}
\end{lemma}

\begin{proof}
The proof is omitted as it is quite standard and follows a similar line of analysis in the proof of Theorem 4.5 in \cite{neely2010stochastic} (see \cite[pp.~92-95]{neely2010stochastic}).
\end{proof}

Lemma~\ref{lem:Lemma_randomized_existed} implies that there exists an optimal $A$-only policy. 
Hence, assuming that the mean reward vector $\bm{\mu}$ is known in advance, one can formulate the reward maximization problem with minimum selection fraction constraint as the following linear program (LP): 
\begin{maxi!}[2]
	{\substack{\mathbf{q}}} {\sum_{\A \in \PN } P_{\mathbf{A}}(\A) \sum_{\S \in \mathcal{S}(\A)} q_{\S}(\A)\sum_{i \in \S}w_i\mu_i \label{optimal objective}}
	{\label{LP}}{}
	\addConstraint{\sum_{\A \in \PN } P_{\mathbf{A}}(\A)\sum_{\S \in \mathcal{S}(\A): i \in \S} q_{\S}(\A)\geq r_i,}{\forall i \in \N, \label{optimal constraint1}}
	\addConstraint{\sum_{\S \in \mathcal{S}(\A)} q_{\S}(\A)= 1,}{\forall \A \in \PN, \label{optimal constraint2}}
	\addConstraint{q_{\S}(\A) \in [0,1],}{ \forall \S \in \mathcal{S}(\A), \forall \A \in \PN. \label{optimal constraint3}}
\end{maxi!}

Suppose that an optimal solution to the above LP is $\mathbf{q^*}=[q^*_{\S}(\A), \forall \S \in \mathcal{S}(\A), \forall \A \in \PN]$. Then an optimal $A$-only policy $\alpha^*$ characterized by $\mathbf{q^*}$ obtains the maximum reward:
\begin{equation}
\label{optimal}
R^* \triangleq \sum_{\A \in \PN } P_{\mathbf{A}}(\A) \sum_{\S \in \mathcal{S}(\A)} q^*_{\S}(\A)\sum_{i \in \S}w_i\mu_i. 
\end{equation}
However, the mean reward vector $\bm{\mu}$ is unknown to the player in advance. Hence, the player not only needs to maximize the reward based on the estimated mean rewards (i.e., exploitation), but she also has to simultaneously learn to obtain a more accurate estimate of the mean rewards (i.e., exploration). Such a learning process typically incurs a loss in the obtained reward, which is called the \emph{regret}. Formally, the time-average regret of a policy $\pi$ for a time horizon of $T$ rounds, denoted by $R_{\pi}(T)$, is defined as the difference between the maximum reward $R^*$ and the expected time-average reward obtained under policy $\pi$ that chooses super arm $S(t)$ in round $t$, i.e.,
\begin{equation}
\begin{aligned}
R_{\pi}(T) &\triangleq R^*- \mathbb{E} \left [ \frac{1}{T} \sum_{t=0}^{T-1} \sum_{i \in S(t)} w_iX_i(t) \right]. 
\label{eq:regret} 
\end{aligned}
\end{equation}
Note that minimizing the regret is equivalent to maximizing the reward. Hence, the regret is a commonly used metric in the MAB literature for measuring the performance of learning algorithms. In this paper, we will adopt the time-average regret defined in \eqref{eq:regret} as the main performance metric.

The key notations used in this paper are listed in Table \ref{notations}.

\section{The LFG Algorithm} \label{sec:approach}


In this section, by carefully integrating the key ideas of UCB \cite{lai1985asymptotically,auer2002finite} and the virtual queue technique \cite{neely2010stochastic}, we develop a new algorithm, called \emph{Learning with Fairness Guarantee (LFG)}, to tackle the CSMAB-F problem. While UCB is extended to deal with the exploitation-exploration tradeoff, the virtual queue technique is employed to handle the fairness constraints.


There are two main challenges in designing an efficient algorithm for the CSMAB-F problem: (i) how to maximize the reward in the face of unknown mean rewards and (ii) how to satisfy the fairness constraints. Note that these two challenges cannot be addressed separately as they are tightly coupled together. Therefore, we need a holistic approach to manage the balance between maximizing the reward and satisfying the fairness constraints. 
In what follows, we will first discuss the key ideas for addressing each individual challenge and then propose the LFG algorithm by carefully integrating them.

The key of maximizing the reward with uncertainty is to strike a balance between exploitation (i.e., choosing the option that gave highest rewards in the past) and exploration (i.e., seeking new options that might give higher rewards in the future). We extend a simple UCB policy based on the concept of optimism in the face of uncertainty to address this challenge and describe the details as follows.

Let $h_i(t)$ be the number of times arm $i$ has been played by the end of round $t$, i.e., $h_i(t) \triangleq \sum_{k=0}^t d_i(k)$.
We set $h_i(-1)=0$ as the system begins at $t=0$.
Also, let $\hat{\mu}_i(t)$ be the sample mean of the observed rewards of arm $i$ by the end of round $t$, i.e., $\hat{\mu}_i(t) \triangleq \frac{\sum_{k=0}^t X_i(k) d_i(k)}{h_i(t)}$. 
We set $\hat{\mu}_i(t)=1$ if arm $i$ has not been played yet by the end of round $t$ (i.e., if $h_i(t)=0$).
We use $\bar{\mu}_i(t)$ to denote the UCB estimate of arm $i$ in round $t$, which is given as follows:
\begin{equation}
\bar{\mu}_i(t) \triangleq
\min \displaystyle{\left\{\hat{\mu}_i(t-1) + \sqrt{\frac{3\log t}{2h_i(t-1)}}, 1 \right\}}, \label{UCB estimation}
\end{equation}
where $\hat{\mu}_i(t-1)$ and $\sqrt{\frac{3\log t}{2h_i(t-1)}}$ correspond to exploitation and exploration, respectively.
We use the above truncated version of the UCB estimate (i.e., capped at 1) as the actual reward must be in $[0,1]$. Similarly, we set $\bar{\mu}_i(t)=1$ if $h_i(t-1)=0$.

\begin{table}[t]
	\centering
	\caption{Summary of Key Notations}
	\label{notations}
	\begin{tabular}{p{0.11\columnwidth}|p{0.78\columnwidth}}
		\hline	
		\textbf{Notations}						&    \textbf{Meaning}			\\
		\hline
		$\N$; $N$ 				& Set of arms; number of arms		\\
		$\PN$				& Power set of $\N$ \\
		$T$						& Time horizon  				\\		
		$m$							& Maximum number of simultaneously played arms \\
		$\mu_i$						& Mean reward of arm $i$ \\
		$w_i$						& Weight of arm $i$ \\
		$r_i$						& Required minimum selection fraction for arm $i$	\\		
		$X_i(t)$					& Reward of arm $i$ in round $t$\\
		$\hat{\mu}_i(t)$			& Sample mean of the observed reward of arm $i$ up to round $t$\\
		$\bar{\mu}_i(t)$ 	& UCB estimate of arm $i$ in round $t$ \\	
		$h_i(t)$				& Number of times arm $i$ has been played up to round $t$ \\
		$d_i(t)$					& Indicator of whether arm $i$ is played or not in round $t$ \\
		$Q_i(t)$					& Virtual queue length for arm $i$ in round $t$\\							
		$A(t)$						& Set of available arms in round $t$ \\
		$S(t)$						& Super arm played in round $t$ \\		
		$P_{\mathbf{A}}(\A)$ 		& Probability that the set of available arms is $\A$ \\
		$\mathcal{S}(\A)$			& Set of feasible super arms when observing available arms $\A$ \\
		$q_{\S}(\A)$					& Probability that an $A$-only policy $\alpha$ chooses super arm $\S$ when observing available arms $\A$ \\
		$\mathcal{C}$	& Maximal feasibility region \\
		$R^*$					& Maximum reward with \emph{a priori} knowledge of $\bm{\mu}$ \\		
		$R_{\pi}(T)$			& Time-average regret of policy $\pi$ \\
		\hline
	\end{tabular}
\end{table}

In the basic MAB setting, the classic UCB policy simply selects the arm that has the largest UCB estimate in each round \cite{lai1985asymptotically,auer2002finite}.
However, in the CSMAB-F setting we are faced with several new challenges introduced by combinatorial arms, availability of arms, and fairness constraints. 
In particular, integrating fairness constraints adds a new layer of difficulty to the combinatorial sleeping MAB problem that is already quite challenging. This is because not only the player is faced with the exploitation-exploration dilemma when attempting to maximize the reward, but she also encounters a new tradeoff between maximizing the reward and satisfying the fairness requirement. Therefore, directly applying the UCB policy will not work as it was designed without fairness constraints in mind. 
Next, we will explain how to use the virtual queue technique to properly handle the fairness constraints, as well as how to cohesively integrate it with UCB to address the overall challenge of the CSMAB-F problem. 

Following the framework developed in \cite{neely2010stochastic}, we create a virtual queue $Q_i$ for each arm $i$ to handle the fairness constraints in \eqref{eq:fraction requirement}. By slightly abusing the notation, we also use $Q_i(t)$ to denote the queue length of $Q_i$ at the beginning of round $t$, which is a counter that keeps track of the ``debt" to arm $i$ up to round $t$. Specifically, the virtual queue length $Q_i(t)$ evolves according to the following dynamics:
\begin{equation}
Q_i(t) = [Q_i(t-1) + r_i - d_i(t-1)]^+ ,  \label{queue}
\end{equation}
where $[x]^+ \triangleq \max\{x, 0\}$. 
We set $Q_i(0)=0$ as the system begins at $t=0$.
As can be seen in the above queue-length evolution, the ``debt" to arm $i$ increases by $r_i$ in each round as $r_i$ is the minimum selection fraction, and it decreases by one if arm $i$ is selected in round $t-1$ (i.e., $d_i(t-1)=1$).

Having introduced the UCB estimate and the virtual queues, we are now ready to describe the proposed LFG algorithm, which is presented in Algorithm~\ref{designed_algorithm}. 
At the very beginning, we initialize $h_i(-1)=0$ and $Q_i(0)=0$ for all arms $i \in \N$~(lines~\ref{alg:init start}-\ref{alg:init end}).
In each round $t$, we first update the UCB estimates $\bar{\mu}_i(t)$ and the virtual queue lengths $Q_i(t)$ according to \eqref{UCB estimation} and \eqref{queue} for all arms $i \in \N$, respectively, based on the decision and the feedback from the previous rounds~(lines~\ref{alg:update mu Q start}-\ref{alg:update mu Q end}); we set $\bar{\mu}_i(t)=1$ if $h_i(t-1)=0$.
Then, we observe the set of available arms $A(t)$~(line~\ref{alg:observe available}) and select a super arm $S(t) \in \mathcal{S}(A(t))$ that maximizes the compound value of the updated $\bar{\mu}_i(t)$ and $Q_i(t)$ as follows~(line~\ref{alg:select}): 
\begin{equation}
S(t) \in \argmax_{\S \in \mathcal{S}(A(t))} \sum_{i \in \S} \left(Q_i(t)+ \eta w_i \bar{\mu}_i(t)\right), 
\label{best super arm}
\end{equation}
where $\eta$ is a positive parameter we can tune to manage the balance between the reward and the virtual queue lengths.
Note that the size of $\mathcal{S}(A(t))$ is exponential in $m$. Hence, the complexity of selecting a super arm $S(t)$ according to \eqref{best super arm} could be prohibitively high in general. However, thanks to the special structure of linear compound reward, we can efficiently solve \eqref{best super arm} and find a best super arm $S(t)$ by iteratively selecting best individual arms. Specifically, we select a super arm $S(t)$ consisting of the top-$m^*$ arms in $A(t)$, where $m^* \triangleq \min\{m, |A(t)|\}$. That is, starting with an empty $S(t)$, we iteratively select arm $i^*$ such that
\begin{equation}
i^* \in \argmax_{i \in A(t) \setminus S(t)} Q_i(t)+ \eta w_i \bar{\mu}_i, \label{best arm}
\end{equation}
and after each iteration, we update super arm $S(t)$ by adding arm $i^*$ to it, i.e., $S(t) = S(t) \cup \{i^*\}$. 
Repeating the above procedure for $m^*$ iterations solves \eqref{best super arm} and finds a best super arm $S(t)$.
After we play arms in $S(t)$ and set vector $\mathbf{d}(t)$ accordingly~(line~\ref{alg:play}), we observe the reward $X_i(t)$ for all played arms $i \in S(t)$~(lines~\ref{alg:reward start}-\ref{alg:reward end}) and update $h_i(t)$ and $\hat{\mu}_i(t)$ accordingly for all arms $i \in \N$~(lines~\ref{alg:update start}-\ref{alg:update end}).



\emph{Remark}: 
As we mentioned earlier, we introduce a design parameter $\eta$ to manage the balance between the reward and virtual queue lengths.
When $\eta$ is large, the LFG algorithm gives a higher priority to maximizing the reward compared to meeting the fairness constraints. This is because an arm with a large estimated reward (i.e., UCB estimate) will be favored, compared to another arm that has a small estimated reward but a large ``debt" (i.e., virtual queue length).
In contrast, when $\eta$ is small, the LFG algorithm gives a higher priority to meeting the fairness constraints because an arm with a large virtual queue length will be favored even if it has a small estimated reward. Indeed, our simulation results presented in Section~\ref{sec:simulation} reveal an interesting tradeoff between the regret and the speed of convergence to a point satisfying the fairness constraints.

Note that the LFG algorithm adopts a linear combination of the virtual queue length and the UCB estimate to address the trade-off between reward maximization and fairness guarantee. The reason that such a natural integration works is partially due to the linearity of the offline problem (i.e., Eq.~\eqref{LP}). In particular, the objective function (i.e., Eq.~\eqref{optimal objective}) is linear because we consider a linear reward function. In the settings with more general nonlinear reward functions, such as a submodular reward function, even the offline problem with known rewards could easily become intractable (e.g., NP-hard) \cite{krause14survey}. In such cases, it remains unclear how to design efficient algorithms that can achieve a good regret performance while satisfying the fairness constraints. We leave this question to our future work.

In addition, our proposed LFG algorithm is based on the drift-plus-penalty approach \cite{neely2010stochastic}. As explained in \cite{neely2010stochastic}, this approach can be viewed as a dual-based approach to the stochastic optimization problem (i.e., the linear program formulated in Eq. \eqref{LP}), and it reduces to the well-known dual subgradient algorithm for linear and convex programs when applied to non-stochastic optimization problems. However, to the best of our knowledge, our work is the first to employ the drift-plus-penalty approach to solve a new MAB problem with fairness constraints. The integration of the virtual queue technique and the UCB algorithm renders the regret analysis more challenging as the traditional regret analysis for the UCB algorithm becomes inapplicable here.

\begin{algorithm}[!t]  
	\caption{Learning with Fairness Guarantee (LFG)}  \label{designed_algorithm}
	\begin{algorithmic}[1]
	    \FOR{$i \in \N$} \label{alg:init start}
	    \STATE Initialize $h_i(-1)=0$ and $Q_i(0)=0$;
	    \ENDFOR \\ \label{alg:init end}
		In each round $t$: 
		\FOR{$i \in \N$}  \label{alg:update mu Q start}
		\IF{$h_i(t-1)>0$}
		\STATE Update $\bar{\mu}_i(t)$ according to \eqref{UCB estimation};
		\ELSE
		\STATE Set $\bar{\mu}_i(t)=1$;
		\ENDIF
		\STATE Update $Q_i(t)$ according to \eqref{queue};
		\ENDFOR  \label{alg:update mu Q end}
		\STATE Observe the set of available arms $A(t)$; \label{alg:observe available}
		\STATE Select super arm $S(t)$ according to \eqref{best super arm}; \label{alg:select}
		\STATE Play arms in $S(t)$ and set vector $\mathbf{d}(t)$ accordingly; \label{alg:play}
        \FOR{$i \in S(t)$} \label{alg:reward start}
		\STATE Observe the reward $X_i(t)$; 
		\ENDFOR \label{alg:reward end}
		\FOR{$i \in \N$} \label{alg:update start}
		\STATE Update $h_i(t)$ and $\hat{\mu}_i(t)$ according to $d_i(t)$ and $X_i(t)$. 
		\ENDFOR \label{alg:update end}
	\end{algorithmic}  
\end{algorithm} 


\section{Main Results} \label{sec:results}
In this section, we analyze the performance of our proposed LFG algorithm and present our main results. Specifically, we show that the LFG algorithm is feasibility-optimal (i.e., it can satisfy any feasible requirement of minimum selection fraction for each individual arm) in Section~\ref{subsec:feasibility} and derive an upper bound on the time-average regret in Section~\ref{subsec:regret}.


\subsection{Feasibility Optimality} \label{subsec:feasibility}
We first present the feasibility-optimality result. That is, the LFG algorithm can satisfy the fairness constraints in \eqref{eq:fraction requirement} for any minimum selection fraction vector $\mathbf{r}$ strictly inside the maximal feasibility region $\mathcal{C}$. 

Note that the constraints in \eqref{eq:fraction requirement} are satisfied as long as the virtual queue system defined in \eqref{queue} is \emph{mean rate stable} \cite[pp. 56-57]{neely2010stochastic}, i.e., $\lim_{T \to \infty} \frac{\mathbb{E}[\sum_{i=1}^N Q_i(T)]}{T}=0$. In our virtual queue system, mean rate stability is implied by a stronger notion called \emph{strong stability}, i.e., $\limsup_{T \to \infty} \frac{1}{T}\sum_{t=0}^{T-1} \mathbb{E}[\sum_{i=1}^N Q_i(t)] < \infty$.
Therefore, in order to prove feasibility-optimality, it is sufficient to show that the virtual queue system is strongly stable whenever the minimum selection fraction vector $\mathbf{r}$ is strictly inside $\mathcal{C}$. We state this result in Theorem~\ref{thm:feasibility}.

\begin{theorem}\label{thm:feasibility}
	The LFG algorithm is feasibility-optimal. Specifically, for any minimum selection fraction $\mathbf{r}$ strictly inside the maximal feasibility region $\mathcal{C}$, the virtual queue system defined in \eqref{queue} is strongly stable under LFG. That is, 
	\begin{equation}
	\limsup_{T\to \infty} \frac{1}{T}\sum_{t=0}^{T-1} \mathbb{E} \left[\sum_{i=1}^N Q_i(t) \right] \leq \frac{B}{\epsilon} < \infty, \label{stability result}
	\end{equation}
	where $B \triangleq \frac{N}{2} + \eta m w_{\max}$ and $\epsilon$ is some positive constant satisfying that $\mathbf{r}+\epsilon\mathbf{1}$ is still strictly inside $\mathcal{C}$, with $\mathbf{1}$ being the $N$-dimensional vector of all ones. 
\end{theorem}

We prove Theorem~\ref{thm:feasibility} by using standard Lyapunov-drift analysis \cite{neely2010stochastic}. The detailed proof is provided in Appendix~\ref{app:proof_feasibility}.

\emph{Remark}:
Note that the work of \cite{cai2018learning} also studies an MAB problem with minimum-guarantee constraints. However, their work differs significantly from ours because their considered minimum guarantee is for the total rewards (of some type/level) rather than for each individual arm, i.e., fairness among arms is not modeled. More importantly, the proposed learning algorithm in \cite{cai2018learning} may violate the constraints. Although they show that the violations are upper bounded by $O(T^{5/6})$, this upper bound implies that the constraints may not be satisfied even after a long enough time. In stark contrast, Theorem~\ref{thm:feasibility} states that our proposed LFG algorithm can satisfy the (long-term) fairness constraints as long as the requirement is feasible. Another difference is that they do not consider sleeping bandits, which can further complicate the problem.


\subsection{Upper Bound on Regret} \label{subsec:regret}
In this subsection, we prove an upper bound on the time-average regret (as defined in \eqref{eq:regret}) under the LFG algorithm. This upper bound is achieved uniformly over time (i.e., for any finite time horizon $T$) rather than asymptotically when $T$ goes to infinity. We state this result in Theorem~\ref{thm:regret}.

\begin{theorem} \label{thm:regret}
Under the LFG algorithm, the time-average regret defined in \eqref{eq:regret} has the following upper bound:
\begin{equation}
R_{\text{LFG}}(T) \leq \frac{N}{2\eta} + \frac{\beta_1 \sqrt{mNT\log T}+ \beta_2N}{T}, \label{eq:regret result}
\end{equation}
where $\beta_1 \triangleq 2\sqrt{6}w_{\max}$, and $\beta_2 \triangleq (1+\frac{5\pi^2}{12}) w_{\max}$.
\end{theorem}

We prove Theorem~\ref{thm:regret} by using a similar line of regret analysis in \cite{hsu2018integrate}. The detailed proof is provided in Appendix~\ref{sec:proof_regret}.


\emph{Remark}:
The derived regret upper bound in \eqref{eq:regret result} is quite appealing as it separately captures the impact of the fairness constraints and the impact of the uncertainty in the mean rewards for any finite time horizon $T$.
Note that the regret upper bound in \eqref{eq:regret result} has two terms. The first term $\frac{N}{2\eta}$ is inversely proportional to $\eta$ and is attributed to the impact of the fairness constraints. Specifically, when $\eta$ is small, the LFG algorithm gives a higher priority to meeting the fairness requirement by favoring an arm with a larger ``debt" (i.e., virtual queue length) as in \eqref{best arm}, even if this arm has a small estimated reward. This results in a larger regret captured in the first term. Similarly, a larger $\eta$ leads to a smaller regret captured in the first term, but it will take longer for the LFG algorithm to converge to a point satisfying the fairness constraints. This interesting tradeoff can also be observed from our simulation results in Section~\ref{sec:simulation}.
The second term $\frac{\beta_1 \sqrt{mNT\log{T}}+ \beta_2 N}{T}$ is of the order $O(\sqrt{\log{T}/T})$. This part of the regret corresponds to the notion of regret in typical MAB problems and is attributed to the cost that needs to be paid in the learning/exploration process. Note that the second term is an \emph{instance-independent} upper bound that does not depend on the problem-specific parameter $\bm{\mu}$. Our derived bound on the time-average regret is consistent with the instance-independent result for basic MAB problems~\cite[Ch. 2.4.3]{bubeck2012regret}\footnote{Time-average regret $O(\sqrt{\log{T}/T})$ vs. cumulative regret $O(\sqrt{T \log{T}})$.}.

\begin{figure*}[!t]
	\begin{minipage}[t]{0.33\textwidth}
		\centering
		\includegraphics[width=1\textwidth]{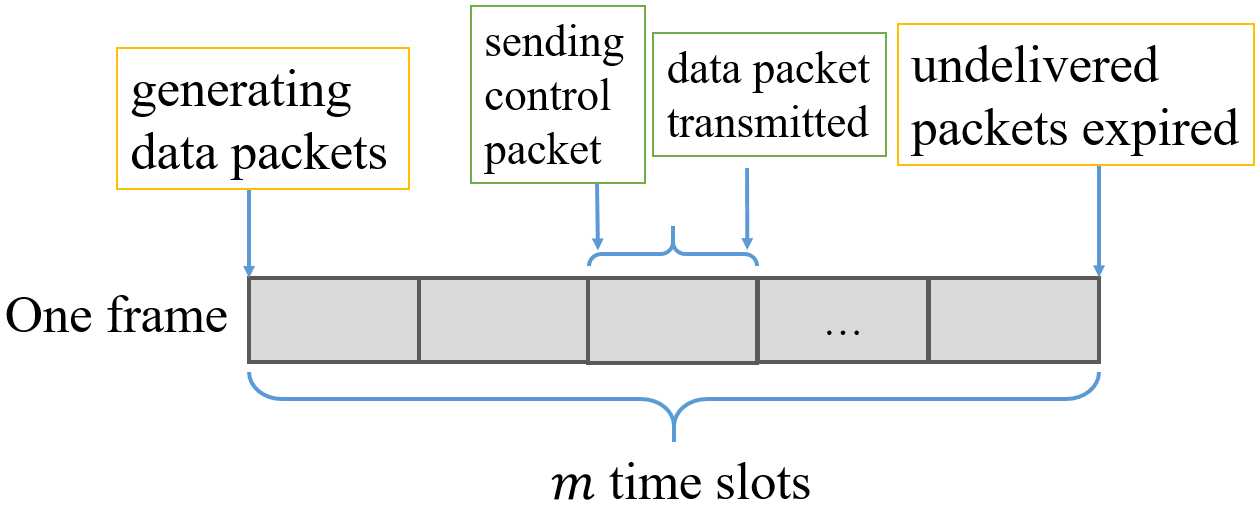}
		\caption{Scheduling of real-time traffic}\label{fig-time}
	\end{minipage}
	\begin{minipage}[t]{0.33\textwidth}
		\centering                                            
		\includegraphics[width=0.85\textwidth]{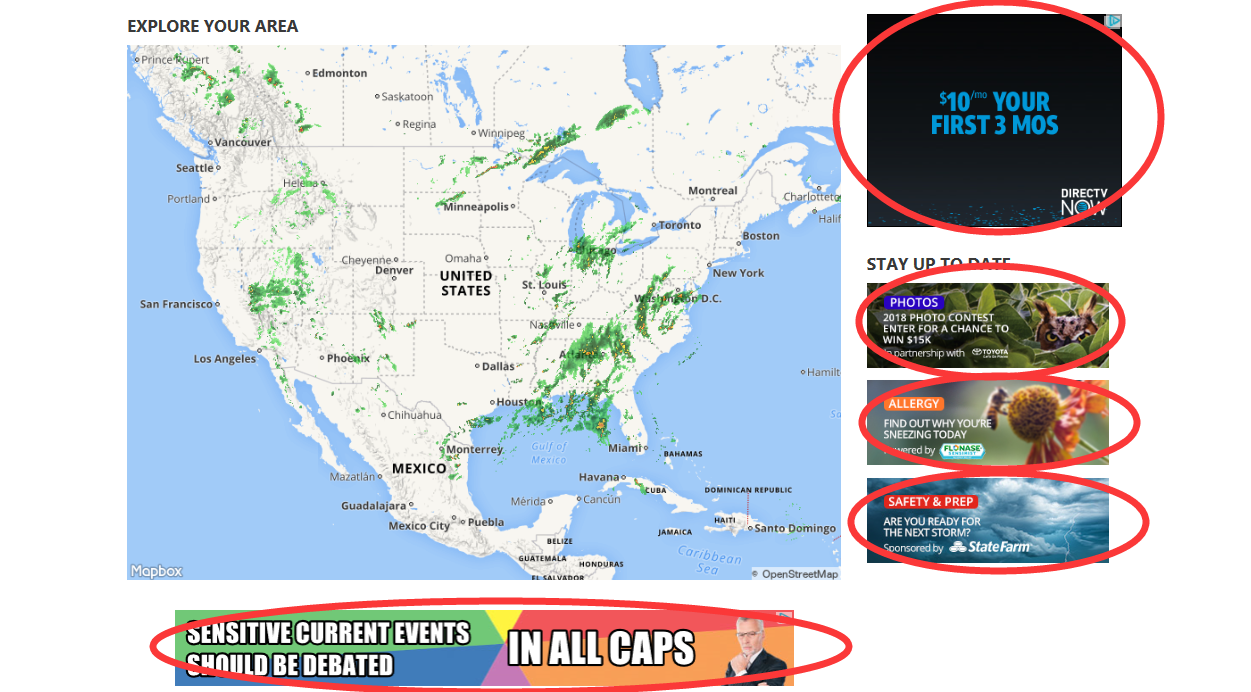}
		\caption{Ad placement}\label{fig:ad-placement}
	\end{minipage}
	\begin{minipage}[t]{0.33\textwidth}
		\centering
		\includegraphics[width=0.8\textwidth]{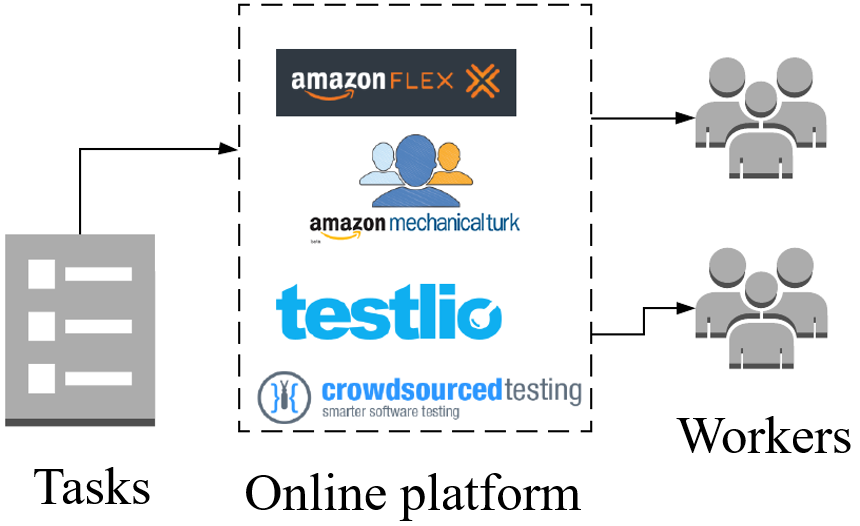}
		\caption{Task assignment in crowdsourcing}\label{fig:crowdsourcing} 
	\end{minipage}
\end{figure*}

\section{Applications} \label{sec:applications}
In this section, we provide more detailed discussions about real-world applications of our proposed CSMAB-F framework. Specifically, we will discuss the following three applications as examples: scheduling of real-time traffic in wireless networks~\cite{hou2009theory}, ad placement in online advertising systems~\cite{AdSpeed}, and task assignment in crowdsourcing platforms~\cite{basik2018fair}. 
               
\subsection{Scheduling of Real-time Traffic in Wireless Networks}
Consider the problem of scheduling real-time traffic with QoS constraints in a single-hop wireless network.
Assume that there are $N$ clients competing for a shared wireless channel to transmit packets to a common AP (see, e.g., \cite{hou2009theory}). Time is slotted. The AP decides which client(s) can transmit at what times. Consider a scheduling cycle, called a frame, that consists of $m$ consecutive time slots. Every client generates one data packet at the beginning of each frame. To avoid interference, we assume that at most one client can transmit in each time slot. Note that some clients may sometimes be unable to transmit when experiencing poor channel conditions (due to fading or mobility). Assume that the channel conditions remain unchanged during a frame but may vary over frames and that the AP obtains the exact knowledge about the channel conditions via probing. At the beginning of each frame, the AP makes scheduling decisions by selecting an available client to transmit in each of the $m$ time slots; at the beginning of each time slot, the AP broadcasts a control packet that announces the scheduling decision, and then, the selected client transmits a packet to the AP in that time slot. We model real-time traffic by assuming that packets have a lifetime of $m$ time slots and expire at the end of the frame. The above framework is illustrated in Fig.~\ref{fig-time}. While a successfully delivered packet will generate a utility, which could represent the value of the information contained in the packet, an expired packet will be dropped at the end of the frame. We assume that the utility corresponding to each client is a random variable, and its mean is unknown \emph{a priori}. There is a weight associated with each client, indicating the importance of the information provided by the client.

The goal of the AP is to maximize the cumulative utilities by scheduling packet transmissions in the face of unknown mean utilities. In addition, each client has a QoS requirement that a minimum delivery ratio must be guaranteed. Clearly, the scheduling problem with minimum delivery ratio guarantee can naturally be formulated as a CSMAB-F problem.

\subsection{Ad Placement in Online Advertising Systems}
Online advertising has emerged as a very popular Internet application~\cite{AdSpeed}. Take a page of \emph{Weather.com} website shown in Fig.~\ref{fig:ad-placement} for example. When an Internet user visits the webpage, the publisher dynamically chooses multiple ads from the ads pool to display in the ad-mix areas (highlighted by red circles in Fig.~\ref{fig:ad-placement}). We assume that the ads pool consists of $N$ ads, and the ad-mix area has a limited capacity, which allows displaying no more than $m$ ads simultaneously. Note that some ads are irrelevant to certain users, depending on the context including users' characteristics (gender, interest, location, etc.) and content of the webpage. Hence, such irrelevant ads can be viewed as unavailable to those users, and the availability of ads depends on the distribution of the context. After seeing a displayed ad, the user may or may not click it. The click-through rate (i.e., the rate at which the ad is clicked) of each ad is unknown \emph{a priori}. Each click of an ad will potentially generate a revenue for the advertiser, which can be viewed as the weight of the ad.

The goal of the ad publisher is to maximize the cumulative revenues by determining a best subset of ads to display in the face of unknown click-through rates. 
In addition, the publisher must guarantee a minimum display frequency for advertisers who pay a fixed cost over a specified period, regardless of users' responses to the displayed ads. Obviously, the ad placement problem with minimum display frequency guarantee fits perfectly into our proposed CSMAB-F framework.

\subsection{Task Assignment in Crowdsourcing Platforms}
The increasing application of crowdsourcing is significantly changing the way people conduct business and many other activities~\cite{basik2018fair}.
Consider a crowdsourcing platform such as Amazon Mechanical Turk, Amazon Flex (for package delivery), and Testlio (for software testing), as shown in Fig.~\ref{fig:crowdsourcing}. 
Tasks arriving to the crowdsourcing platform will be assigned to a group of workers with different unknown skill levels.
Specifically, when a task arrives, the platform may divide the task into multiple sub-tasks; then the sub-tasks will be assigned to no more than $m$ workers from a pool of $N$ workers, due to the number of sub-tasks or a limited budget. 
Note that some workers could be unavailable to take certain tasks due to various reasons (time conflicts, location constraints, limited skills, preferences, etc.). 
Each completed task will generate a payoff that depends on the quality or efficiency of the workers. The payoff is a random variable, and its mean is unknown \emph{a priori} due to unknown skill levels of workers.

The goal of the crowdsourcing platform is to maximize the cumulative payoffs by determining an optimal task allocation in the face of unknown mean payoffs. In addition, the platform has to take fairness towards workers into account through a minimum assignment ratio guarantee for each worker. This fairness guarantee helps maintain a healthy and sustainable platform with improved worker satisfaction and higher worker participation. Apparently, our proposed CSMAB-F framework can be applied to address the task assignment problem with minimum assignment ratio guarantee.

\begin{figure}[!t] 
	\begin{minipage}[t]{0.5\textwidth}
	\centering
	\subfigure[Time-average Regret]{
		\label{fig:Alg_compare:regret} 
		\includegraphics[width=0.7\textwidth]{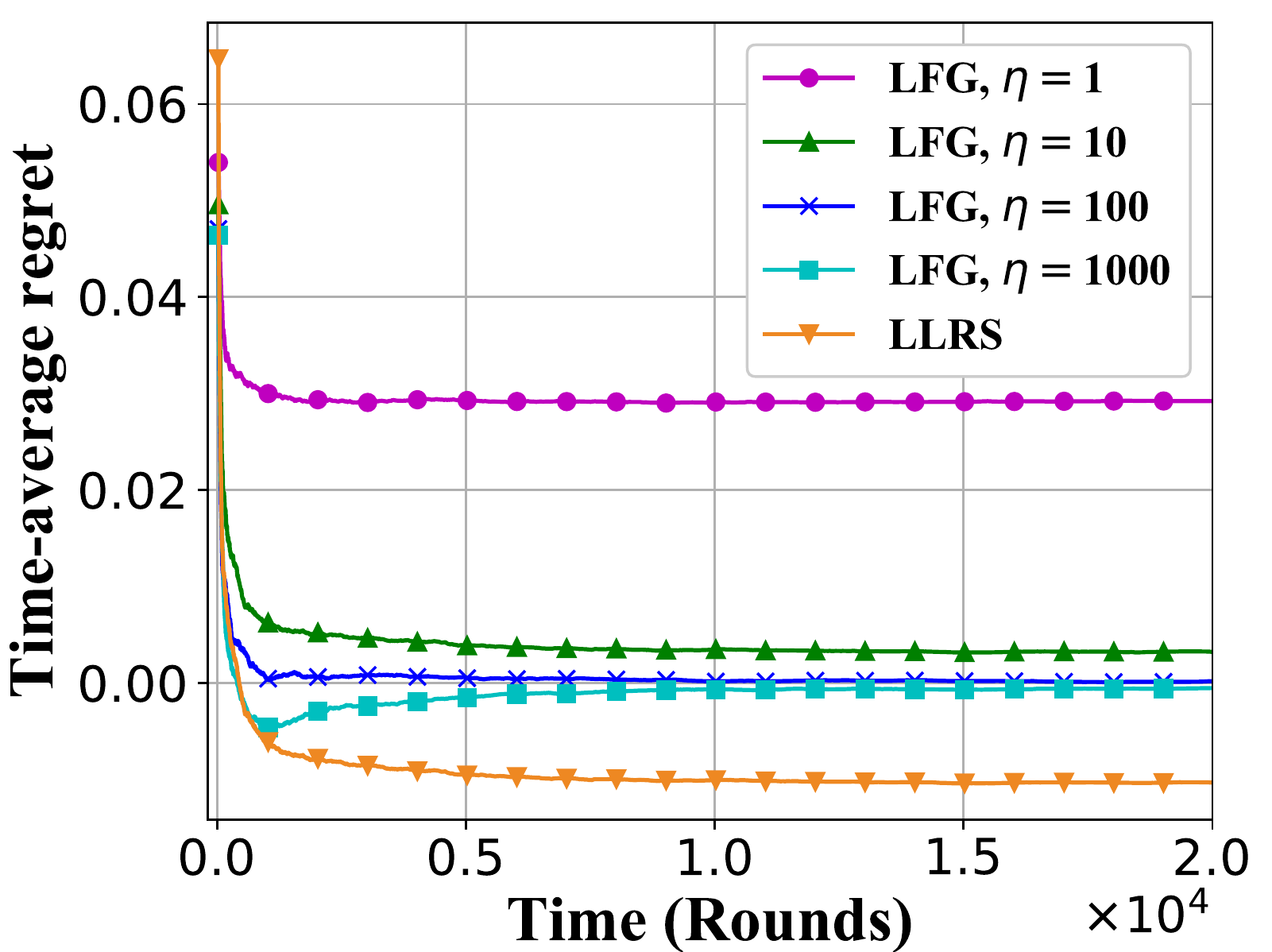}}
	\subfigure[Selection fraction]{
		\label{fig:Alg_compare:fraction} 
		\includegraphics[width=0.7\textwidth]{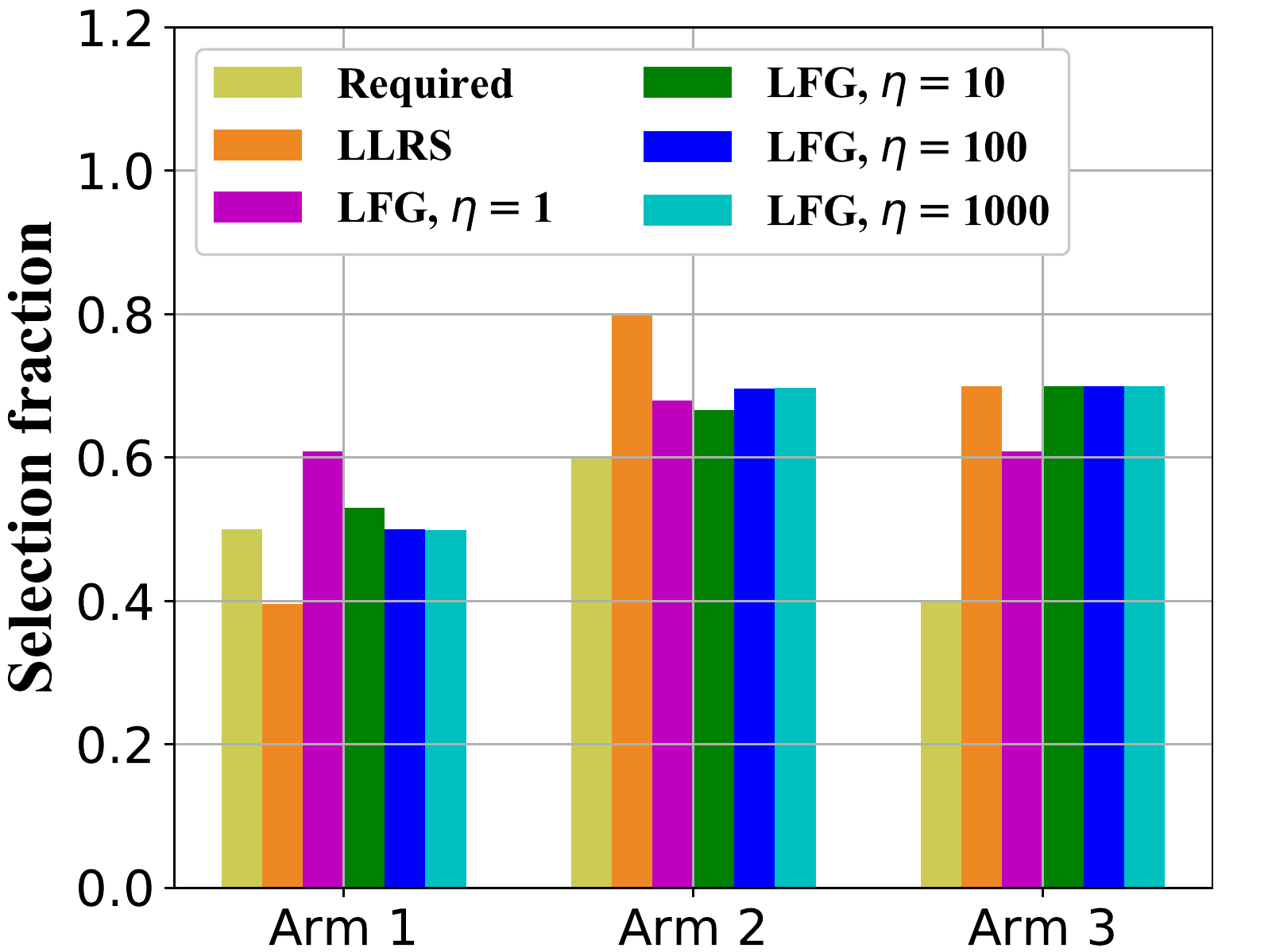}}
	\caption{Performance comparisons of different algorithms}
	\label{fig:Alg_compare} 
	\end{minipage}
\end{figure}

			

\section{Numerical Results} \label{sec:simulation}
In this section, we conduct simulations to evaluate the performance of our proposed LFG algorithm and discuss several interesting observations based on the simulation results.

We consider two scenarios for the simulations: (i) $N=3$ and $m=2$; (ii) $N=10$ and $m=6$. Since the observations are similar for these two scenarios, we will focus on the discussion about the first scenario due to space limitations. We assume that the availability of arm $i$ is a binary random variable that is \emph{i.i.d.} over time with mean $p_i$.
Then, the distribution of available arms can be computed as 
$P_{\mathbf{A}}(\A)=\prod_{i\in \A}p_i \prod_{i \notin \A}(1-p_i)$
for all $\A \in \mathcal{P}(\N)$. 
We also assume binary rewards with the same unit weight (i.e., $w_i=1$) for all the arms.
The detailed setting of other parameters is as follows: $\bm{\mu} = (0.4, 0.5, 0.7)$, $\mathbf{r} = (0.5, 0.6, 0.4)$, and $\mathbf{p} = (p_1, p_2, p_3) = (0.9, 0.8, 0.7)$.


\begin{figure*}[!t]
	\begin{minipage}[t]{0.99\textwidth}
	\centering
	\subfigure[Arm $1$]{
		\label{fig:trade-off: arm1} 
		\includegraphics[width=0.30\textwidth]{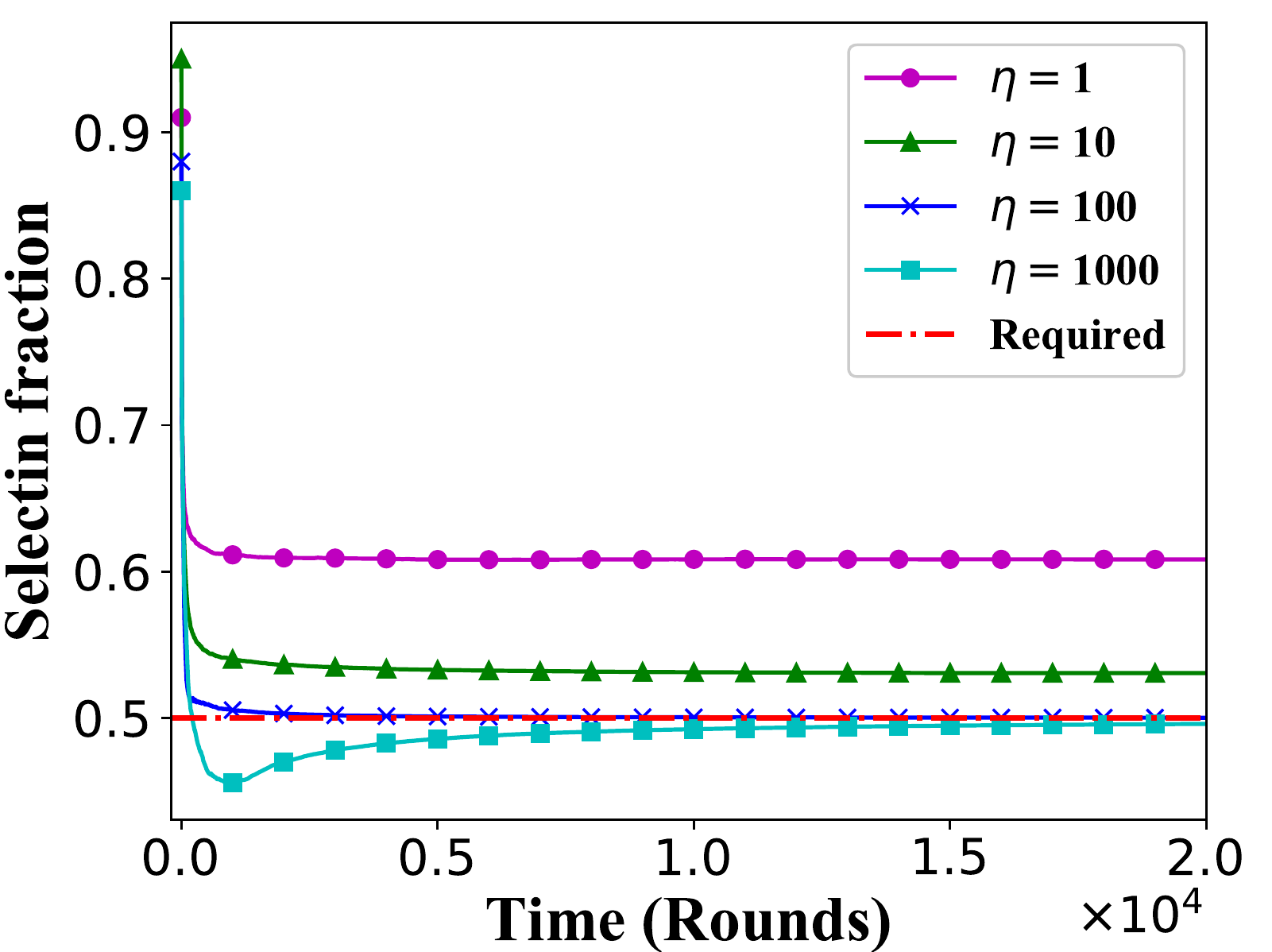}}
	\quad
	\subfigure[Arm $2$]{
		\label{fig:trade-off: arm2} 
		\includegraphics[width=0.30\textwidth]{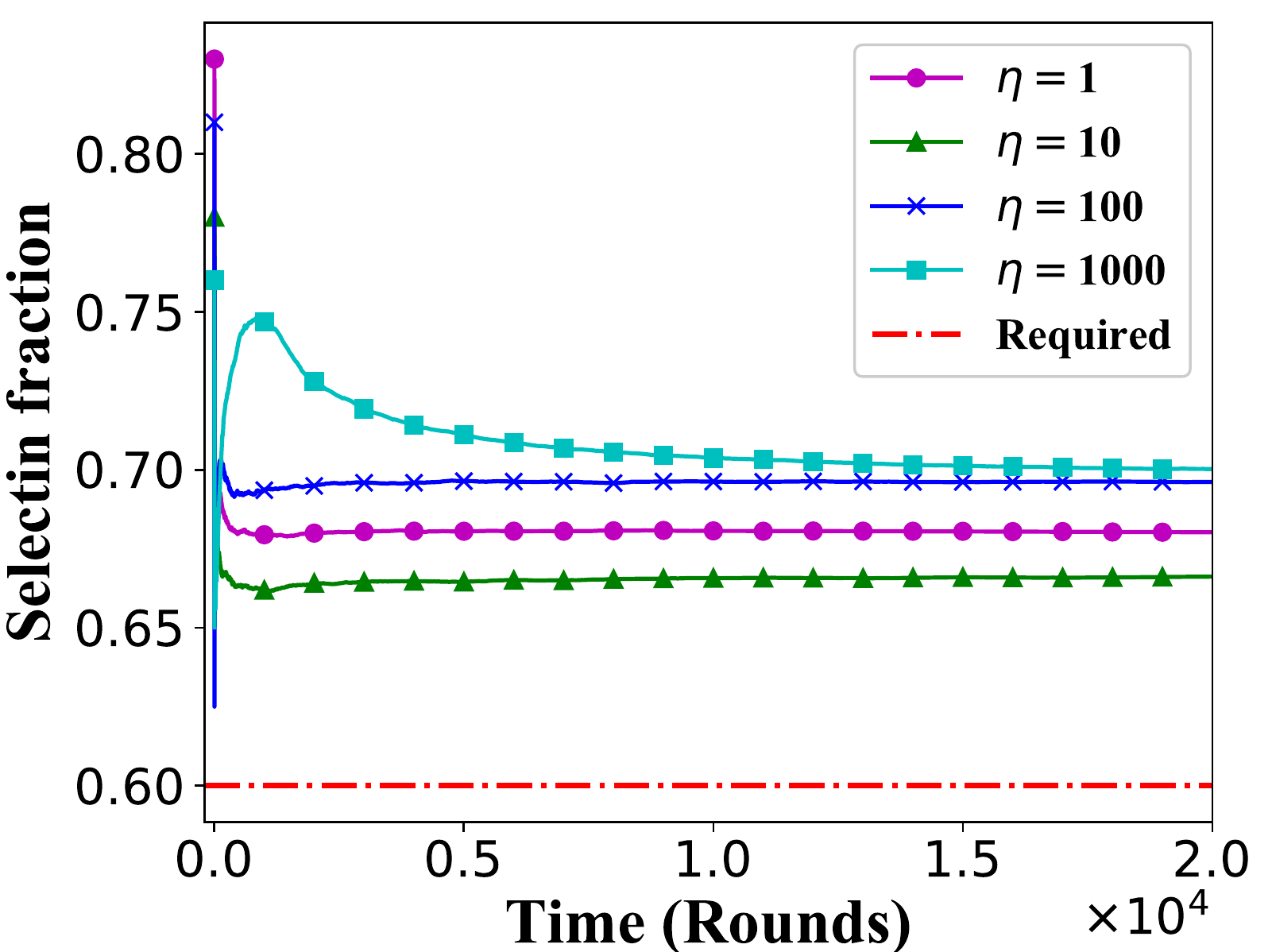}}
	\quad
	\subfigure[Arm $3$]{
		\label{fig:trade-off: arm3} 
		\includegraphics[width=0.30\textwidth]{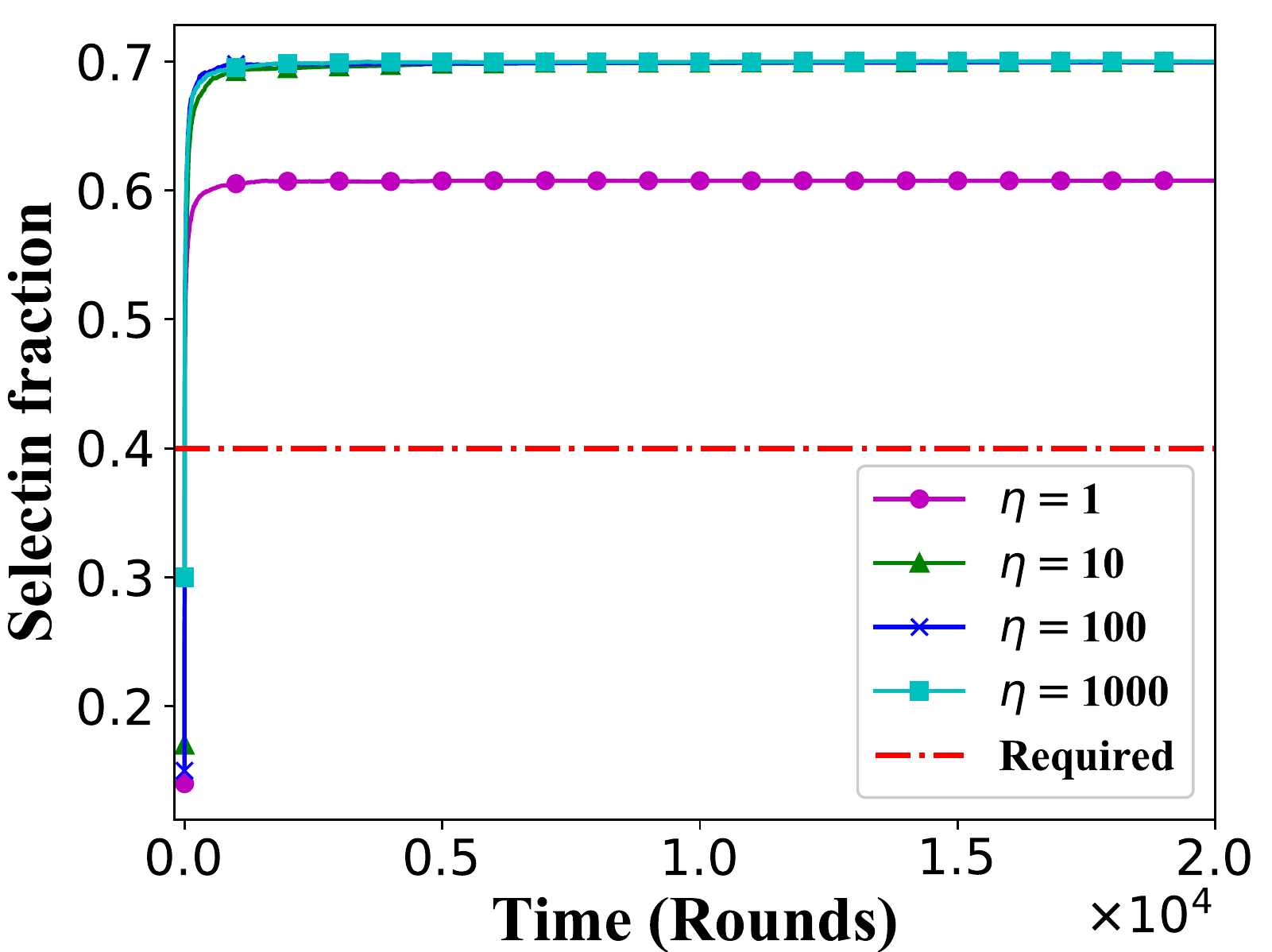}}
	\caption{Selection fraction over time under LFG with different values of $\eta$}
	\label{fig:trade-off} 
	\end{minipage}
\end{figure*}


First, in order to demonstrate that LFG can effectively meet the fairness requirement, we compare LFG with a fairness-oblivious combinatorial MAB algorithm, called \emph{Learning with Linear Rewards (LLR)} {\cite{gai2012combinatorial}}. We modify the LLR algorithm to accommodate sleeping bandits; the modified version is called \emph{LLR for Sleeping bandits (LLRS)}.
In each round $t$, observing the set of available arms $A(t)$, LLRS selects a super arm $S(t)$ that has the largest weighted sum of the UCB estimates among all the feasible super arms in $\mathcal{S}(A(t))$, i.e., $S(t) \in \argmax_{\S \in \mathcal{S}(A(t))} \sum_{i \in \S} w_i \bar{\mu}_i(t)$. 
Note that LLRS is oblivious of the fairness constraints in \eqref{eq:fraction requirement}.

We simulate LFG with $\eta \in \{1, 10, 100, 1000\}$ and LLRS for $T = 2 \times 10^4$ rounds (at which all the considered algorithms are observed to converge) and present the results in Fig.~\ref{fig:Alg_compare}. 
Fig.~\ref{fig:Alg_compare:regret} shows the time-average regret over time for the considered algorithms; Fig.~\ref{fig:Alg_compare:fraction} shows the selection fraction of each arm at the end of the simulation (i.e., at $T = 2 \times 10^4$). From Fig.~\ref{fig:Alg_compare:regret}, we can make the following observations: (i) LFG with a larger $\eta$ results in a smaller regret, and LFG with $\eta \ge100$ approaches a zero regret; (ii) LLRS achieves the smallest regret, which is even negative (i.e., it achieves a reward larger than the optimal $R^*$). Observation (i) is expected, as we explained in Section~\ref{subsec:regret}: the upper bound on regret in \eqref{eq:regret result} approaches zero when both $\eta$ and $T$ become large. Observation (ii) is not surprising because LLRS is fairness-oblivious and may produce an infeasible solution. Indeed, Fig.~\ref{fig:Alg_compare:fraction} shows that Arm~1's selection fraction under LLRS is smaller than the required value (0.4 vs. 0.5). This is because Arm 1 has the smallest mean reward and is not favored under LLRS, which is unaware of the fairness contraints.
On the other hand, Fig.~\ref{fig:Alg_compare:fraction} also shows that with different values of $\eta$, LFG consistently satisfies the required minimum selection fraction, which verifies our theoretical result on feasibility-optimality of LFG (Theorem~\ref{thm:feasibility}).

At first glance, the above observations seem to suggest that LFG with a large $\eta$ is desirable because that leads to a vanishing regret while still providing fairness guarantee. However, what is missing here is the speed of convergence to a point satisfying the fairness requirement, which is another critical design concern in practice. To understand the convergence speed of LFG with different values of $\eta$, in Fig.~\ref{fig:trade-off} we plot the selection fraction over time for each arm. Taking Fig.~\ref{fig:trade-off: arm1} for example, we can observe that the convergence slows down as $\eta$ increases. In addition, before LFG with $\eta=1000$ converges (e.g., when $T \le 10^4$), the actual selection fraction of Arm 1 does not meet the required minimum value of 0.5. Since the constraints may be temporarily violated, the regret could even be negative before LFG converges (see $\eta=1000$ in Fig.~\ref{fig:Alg_compare:regret}). Therefore, the simulation results reveal an interesting tradeoff between the regret and the convergence speed. We can control and optimize this tradeoff by tuning $\eta$. For example, for the considered scenario, LFG with $\eta=100$ seems to achieve a good balance between the regret and the convergence speed.

\begin{figure}[!t]
		\centering
		\includegraphics[width=0.7\linewidth]{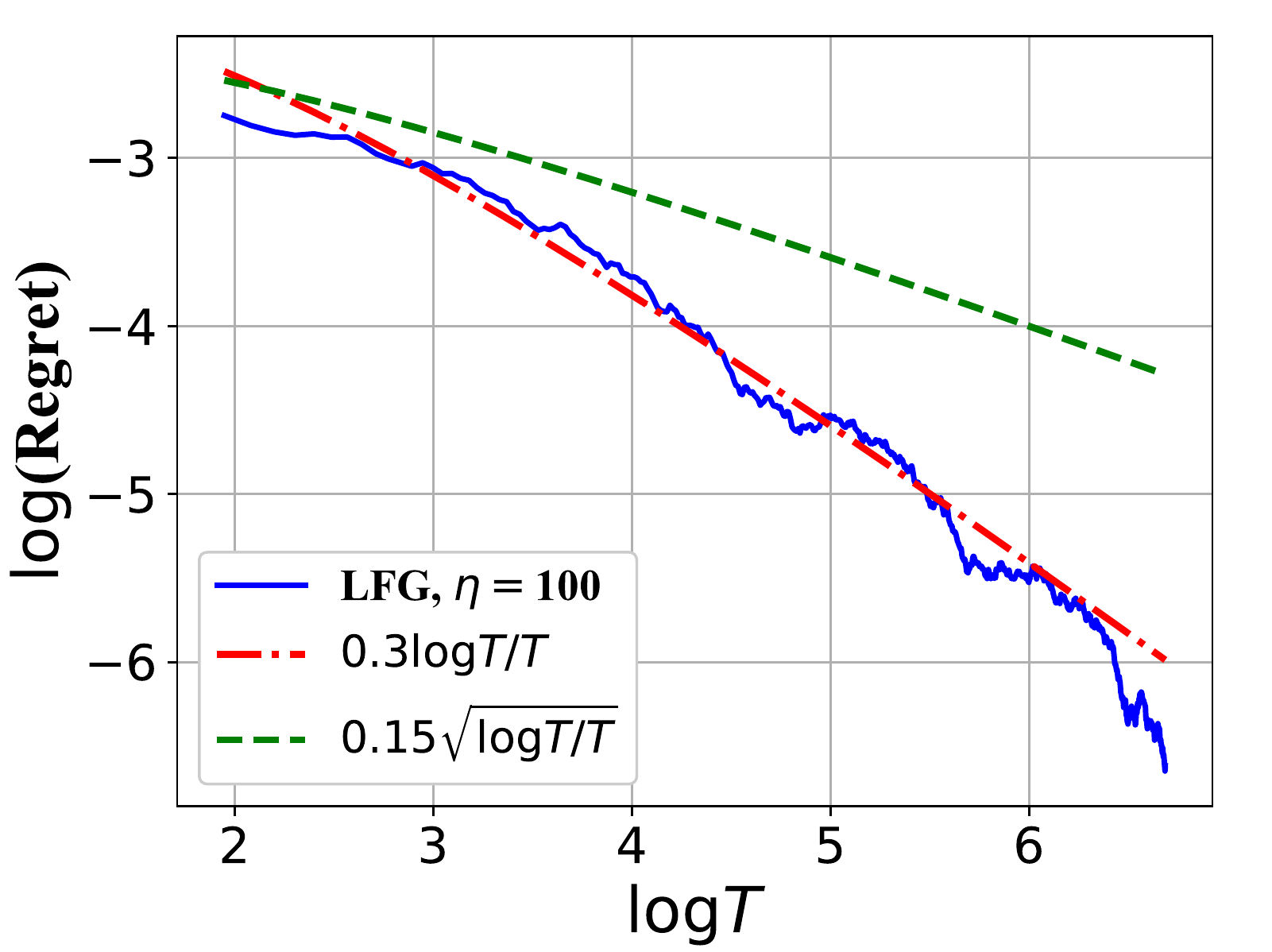}
		\caption{Regret vs. $T$}\label{fig:regret_time}
\end{figure}

Finally, we want to investigate the tightness of the upper bound derived in \eqref{eq:regret result}. Consider the average of 100 independent simulation runs for LFG with $\eta=100$. Fig.~\ref{fig:regret_time} shows the time-average regret vs. the time horizon $T$ in a log-log plot. Recall that the upper bound in \eqref{eq:regret result} has two terms. The impact of $T$ appears in the second term that is of the order $\sqrt{\log{T}/T}$. When $T$ becomes large, it becomes difficult to see the impact of $T$ on the regret as the first term $\frac{N}{2\eta}$ becomes dominant. Therefore, we consider the region with $T \le 1000$ (i.e., $\log{T} \le 6.9$). Fig.~\ref{fig:regret_time} seems to suggest that the time-average regret follows the order $\log{T}/T$ rather than $\sqrt{\log{T}/T}$. This implies that the upper bound in \eqref{eq:regret result} is not tight. One reason could be that the $\sqrt{\log{T}/T}$ bound is instance-independent. It remains open whether one can come up with novel analytical techniques to derive a better bound of $\log{T}/T$.

\section{Conclusion} \label{sec:conclusion}
In this paper, we proposed a unified CSMAB-F framework that integrates several critical factors (i.e., combinatorial actions, availability of actions, and fairness) of the system in many real-world applications. 
In particular, no prior work has studied MAB problems with fairness constraints on a minimum selection fraction for each individual arm. To address the new challenges introduced by modeling these factors, we developed a new LFG algorithm that achieves a provable regret upper bound while effectively providing fairness guarantee.

We leave the following interesting questions to our future work: Can one prove a tighter upper bound on regret? How to develop efficient algorithms for a more general model that potentially accounts for nonlinear reward functions, more sophisticated combinatorial structures (e.g., matroids), and more general fairness criteria other than temporal fairness that we consider in this paper?


\section{ACKNOWLEDGMENTS}
The  authors  would  like  to  thank  Prof.  R.  Srikant  at  University of Illinois at Urbana-Champaign and Prof. Bin Li at University of Rhode Island for valuable comments on this paper.

\bibliographystyle{IEEEtran}
\bibliography{reference}

\appendix

\subsection{Proof of Theorem~\ref{thm:feasibility}} \label{app:proof_feasibility}

\begin{proof}
Consider the LFG algorithm. To prove feasibility optimality, we want to show that for any vector $\mathbf{r}$ strictly inside the maximal feasibility region $\mathcal{C}$, the minimum selection fraction requirements (i.e., Eq.~\eqref{eq:fraction requirement}) are satisfied. Note that the requirements of \eqref{eq:fraction requirement} are satisfied as long as the virtual queue system defined in \eqref{queue} is \emph{mean rate stable} \cite[pp. 56-57]{neely2010stochastic}, i.e., $\lim_{T \to \infty} \frac{\mathbb{E}[\sum_{i=1}^N Q_i(T)]}{T}=0$. In our virtual queue system, mean rate stability is implied by a stronger notion called \emph{strong stability}, i.e., $\limsup_{T \to \infty} \frac{1}{T}\sum_{t=0}^{T-1} \mathbb{E}[\sum_{i=1}^N Q_i(t)] < \infty$.
Therefore, it is sufficient to show that the virtual queue system is strongly stable for any vector $\mathbf{r}$ strictly inside $\mathcal{C}$.

We proceed the proof using the Lyapunov-drift analysis~\cite{neely2010stochastic}. 
Let $\mathbf{Q}(t)=(Q_1(t),\dots, Q_N(t))$ be the queue-length vector in round $t$. Consider the following Lyapunov function:
\begin{equation}
L(\mathbf{Q}(t)) \triangleq \frac{1}{2} \sum_{i=1}^N Q_i^2(t).
\end{equation}
The drift of the Lyapunov function is given by
\begin{equation}
\begin{aligned}
&L(\mathbf{Q}(t+1))-L(\mathbf{Q}(t)) \\
&= \frac{1}{2}\sum_{i=1}^{N}{Q_i^2(t+1)} -  \frac{1}{2} \sum_{i=1}^{N}{Q_i^2(t)} \\
&\overset{(a)}{\leq} \frac{1}{2}\sum_{i=1}^{N}(Q_i(t)+r_i-d_i(t))^2 -  \frac{1}{2} \sum_{i=1}^{N}{Q_i^2(t)} \\
&= \frac{1}{2} \sum_{i=1}^{N}{(r_i-d_i(t))^2}+ \sum_{i=1}^{N}{(r_i-d_i(t))Q_i(t)} \\
&\overset{(b)}{\leq} \frac{N}{2} + \sum_{i=1}^{N}{r_iQ_i(t)} -\sum_{i=1}^{N}{d_i(t)Q_i(t)}, \\ \label{eq:interval drift}
\end{aligned}
\end{equation}
where $(a)$ is from the queue-length evolution \eqref{queue} and $(b)$ holds because both $r_i$ and $d_i(t)$ are within $[0, 1]$. 
Taking conditional expectation of both sides of the above gives
\begin{equation}
\begin{aligned}
&\mathbb{E}[L(\mathbf{Q}(t+1))-L(\mathbf{Q}(t))|\mathbf{Q}(t)]  \\
&\le \frac{N}{2} + \sum_{i=1}^{N}{r_iQ_i(t)} - \mathbb{E} \left[ \sum_{i=1}^{N}{d_i(t)Q_i(t)} |\mathbf{Q}(t) \right] \\
&= \frac{N}{2} + \sum_{i=1}^{N}{r_iQ_i(t)} - \mathbb{E} \left[ \sum_{i\in S(t)}{d_i(t)Q_i(t)} |\mathbf{Q}(t) \right] \\
&= \frac{N}{2} + \sum_{i=1}^{N}{r_iQ_i(t)}  +\mathbb{E}\left[\sum_{i \in S(t)}{\eta w_i\bar{\mu}_i(t)}|\mathbf{Q}(t)\right]  \\
&\quad -\mathbb{E}\left[\sum_{i \in S(t)}{\left(Q_i(t)+\eta w_i\bar{\mu}_i(t)\right)}|\mathbf{Q}(t)\right] \\
&\leq \frac{N}{2} + \sum_{i=1}^{N}{r_iQ_i(t)}+\eta m w_{\max} \\
&\quad -\mathbb{E}\left[\sum_{i \in S(t)}{\left(Q_i(t)+\eta w_i\bar{\mu}_i(t)\right)}|\mathbf{Q}(t)\right] \\
&= B +\sum_{i=1}^{N}{r_iQ_i(t)} \\ &-\mathbb{E}\left[\sum_{i \in S(t)}{\left(Q_i(t)+\eta w_i\bar{\mu}_i(t)\right)}|\mathbf{Q}(t) \right], \label{eq:drift}
\end{aligned}
\end{equation}
where the last inequality holds because $w_i \leq w_{\max}$, $\bar{\mu}_i(t) \leq 1$, and $|S(t)|\leq m$, and $B \triangleq \frac{N}{2} + \eta m w_{\max}$ is a constant.

Recall that $\mathbf{r}$ is strictly inside $\mathcal{C}$. Then, there must exist some $\epsilon>0$ such that $\mathbf{r} + \epsilon \mathbf{1}$ is also strictly inside $\mathcal{C}$, where $\mathbf{1}$ denotes the $N$-dimensional all-ones vector. 
By Lemma~\ref{lem:Lemma_randomized_existed}, there exists an $A$-only policy $\alpha$ that can support vector $\mathbf{r}+\epsilon \mathbf{1}$. That is, 
\begin{equation}
	\sum_{\A \in \PN } P_{\mathbf{A}}(\A)\sum_{\S \in \mathcal{S}(\A): i \in \S} q^{\alpha}_{\S}(\A) \geq r_i+\epsilon,~\forall i \in \N, \label{eq:random_feasibility}
\end{equation}  
where $\mathbf{q}^{\alpha} = [q_{\S}^{\alpha}(\A), \forall \S \in \mathcal{S}(\A), \forall \A \in \PN]$ is the group of probability distributions associated with policy $\alpha$.
Recall that in each round $t$, policy $\alpha$ observes available arms $A(t)$ and chooses a super arm $S^{\alpha}(t) \in \mathcal{S}(A(t))$ independent of $\mathbf{Q}(t)$.
Then, the last term of the right-hand side of \eqref{eq:drift} satisfies
\begin{equation}
	\begin{aligned}
	&\mathbb{E}\left[\sum_{i \in S(t)}{\left(Q_i(t)+\eta w_i\bar{\mu}_i(t)\right)}|\mathbf{Q}(t)\right]  \\
	&=  \mathbb{E}\left[ \mathbb{E}\left[\sum_{i \in S(t)}{\left(Q_i(t)+\eta w_i\bar{\mu}_i(t)\right)}|\mathbf{Q}(t),A(t)\right] \right] \\
	&\overset{(a)}{\geq} \mathbb{E}\left[	 \mathbb{E}\left[\sum_{i \in S^{\alpha}(t)}{\left(Q_i(t)+\eta w_i\bar{\mu}_i(t)\right)}|\mathbf{Q}(t),A(t)\right] \right] \\ 
	&\geq \mathbb{E}\left[ \mathbb{E}\left[\sum_{i \in S^{\alpha}(t)}{Q_i(t)}|\mathbf{Q}(t),A(t)\right] \right] \\
	&\overset{(b)}{=} \mathbb{E}\left[ \mathbb{E}\left[\sum_{i \in S^{\alpha}(t)}{Q_i(t)}|A(t)\right] \right] \\
	&\overset{(c)}{=} \mathbb{E}\left[ \sum_{\S \in \mathcal{S}(A(t))} q^{\alpha}_{\S}(A(t))\sum_{i \in \S}{Q_i(t)} \right] \\
	&= \sum_{\A \in \PN}P_{\mathbf{A}}(\A) \sum_{\S \in \mathcal{S}(\A)} q^{\alpha}_{\S}(\A)\sum_{i \in \S}{Q_i(t)} \\
	&= \sum_{i=1}^{N}Q_i(t)\sum_{\A \in \PN}P_{\mathbf{A}}(\A)\sum_{\S \in \mathcal{S}(\A): i \in \S} q^{\alpha}_{\S}(\A),
	\end{aligned} 
	\label{eq:LFG_better_optimal}
\end{equation}
where $(a)$ is due to the operations of LFG (specifically, \eqref{best super arm}), $(b)$ holds because policy $\alpha$'s decision is independent of $\mathbf{Q}(t)$, and $(c)$ is due to the operations of policy $\alpha$.  
	
Substituting \eqref{eq:LFG_better_optimal} into  \eqref{eq:drift} and applying \eqref{eq:random_feasibility} give
\begin{equation}
\begin{aligned}
&\mathbb{E}[L(\mathbf{Q}(t+1))-L(\mathbf{Q}(t))|\mathbf{Q}(t)] \\
&\leq B + \sum_{i=1}^{N}{r_iQ_i(t)} \\ &\quad -\sum_{i=1}^{N}Q_i(t)\sum_{\A \in \PN}P_{\mathbf{A}}(\A)\sum_{\S \in \mathcal{S}(\A): i \in \S} q^{\alpha}_{\S}(\A) \\
&\leq B + \sum_{i=1}^{N}{r_iQ_i(t)}- \sum_{i=1}^N{Q_i(t)}(r_i+\epsilon) \\ \label{eq:drift-result}
&= B - \epsilon\sum_{i=1}^N{Q_i(t)}.
\end{aligned}
\end{equation}

Finally, invoking the Lyapunov Drift Theorem {\cite[Theorem~4.1]{neely2010stochastic}} gives 
\eqref{stability result}, which completes the proof.
\end{proof}

\subsection{Proof of Theorem~\ref{thm:regret}} \label{sec:proof_regret}

\begin{proof}
Consider an optimal $A$-only policy $\alpha^*$ and its associated probability distributions $\mathbf{q}^* = [q_{\S}^*(\A), \forall \S \in \mathcal{S}(\A), \forall \A \in \PN]$. Let $S^*(t)$ be the super arm selected by policy $\alpha^*$ in round $t$. Vector $\mathbf{d}^*(t)=(d^*_1(t),\dots, d^*_N(t))$ is the corresponding action vector. Due to \eqref{optimal}, we have
\begin{equation}
\begin{aligned}
R^* &= \sum_{\A \in \mathcal{P}(\N) } P_{\mathbf{A}}(\A) \sum_{\S \in \mathcal{S}(\A)} q^*_{\S}(\A)\sum_{i \in \S}w_i\mu_i \\
&=\mathbb{E}\left[\sum_{i\in S^*(t)} w_i{\mu}_i\right]. 
\label{eq:optimal_reward}
\end{aligned}
\end{equation}

Plugging \eqref{eq:optimal_reward} into \eqref{eq:regret}, we can rewrite the regret of LFG as
\begin{equation}
\begin{aligned}
R_{\text{LFG}}(T) &= R^* - \frac{1}{T}  \mathbb{E}\left[\sum_{t=0}^{T-1}\sum_{i\in S(t)} w_iX_i(t)\right] \\  
&=  \frac{1}{T} \sum_{t=0}^{T-1} \left\{R^*- \mathbb{E}\left[\sum_{i\in S(t)} w_i{\mu}_i\right]\right\}\\
&= \frac{1}{T} \sum_{t=0}^{T-1} \mathbb{E}\left[ \vphantom{\sum_{i\in S^*(t)} w_i{\mu}_i-\sum_{i\in S(t)} w_i{\mu}_i} \right.
\underbrace{\sum_{i\in S^*(t)} w_i{\mu}_i-\sum_{i\in S(t)} w_i{\mu}_i}_{\Delta R(t)} 
\left. \vphantom{\sum_{i\in S^*(t)} w_i{\mu}_i-\sum_{i\in S(t)} w_i{\mu}_i}\right].
\label{eq:regret_sum}
\end{aligned}
\end{equation}
We define the following quantity:
\begin{equation}
\begin{aligned}
\Delta R(t) &\triangleq \sum_{i\in S^*(t)}w_i{\mu}_i - \sum_{i\in S(t)}w_i{\mu}_i \\
&= \sum_{i=1}^{N}w_i{\mu}_id_i^*(t) - \sum_{i=1}^{N}w_i{\mu}_id_i(t),
\end{aligned}
\end{equation}
which captures the gap between the expected rewards achieved by policy $\alpha^*$ and LFG in round $t$. 
Adding $\Delta R(t)$ scaled by $\eta$ to the drift of the Lyapunov function (i.e., \eqref{eq:interval drift}) gives the drift-plus-regret:
\begin{equation}
\begin{aligned}
&L(\mathbf{Q}(t+1)) - L(\mathbf{Q}(t)) + \eta \Delta R(t)\\
&\leq \frac{N}{2} + \sum_{i=1}^{N}{r_iQ_i(t)} -\sum_{i=1}^{N}{d_i(t)Q_i(t)} \\
&\quad + \eta \sum_{i=1}^{N}w_i{\mu}_id_i^*(t)- \eta \sum_{i=1}^{N}w_i{\mu}_id_i(t)\\
&= \frac{N}{2} +\sum_{i=1}^{N}{(Q_i(t)+\eta w_i{\mu}_i)(d_i^*(t)-d_i(t))}\\
&\quad + \sum_{i=1}^{N}{Q_i(t)(r_i-d_i^*(t))}.
\label{eq:drift_plus_regret}
\end{aligned}
\end{equation}
We can bound the expected drift-plus-regret as
\begin{equation}
\begin{aligned}
&\mathbb{E}[L(\mathbf{Q}(t+1)) - L(\mathbf{Q}(t)) + \eta \Delta R(t)] \\
&\leq \frac{N}{2} +\sum_{i=1}^{N}\mathbb{E}[{(Q_i(t)+\eta w_i{\mu}_i)(d_i^*(t)-d_i(t))}] \\
&\quad + \sum_{i=1}^{N}\mathbb{E}[{Q_i(t)(r_i-d_i^*(t))}]\\
&\leq \frac{N}{2} +\mathbb{E}\left[\vphantom{\sum_{i=1}^{N}{(Q_i(t)+\eta w_i{\mu}_i)(d_i^*(t)-d_i(t))}} \right.
\underbrace{\sum_{i=1}^{N}{(Q_i(t)+\eta w_i{\mu}_i)(d_i^*(t)-d_i(t))}}_{C_1(t)} 
\left. \vphantom{\sum_{i=1}^{N}{(Q_i(t)+\eta w_i{\mu}_i)(d_i^*(t)-d_i(t))}} \right], 
\end{aligned}
\label{eq:expected_drift_plus_regret}
\end{equation}
where the last step follows from $\mathbb{E}[Q_i(t)d_i^*(t)] = \mathbb{E}[Q_i(t)]\mathbb{E}[d_i^*(t)]$ (due to the decision of policy $\alpha^*$ being independent of the queue length $Q_i(t)$) and $\mathbb{E}[d_i^*(t)] \geq r_i$ (because policy $\alpha^*$ is stationary and feasible). Define $C_1(t) \triangleq \sum_{i=1}^{N}{(Q_i(t)+\eta w_i{\mu}_i)(d_i^*(t)-d_i(t))} $. Summing \eqref{eq:expected_drift_plus_regret} for all $t \in \{0,\dots, T-1\}$, using the trick of telescoping sum, and dividing both sides of the inequality by $T \eta$, we obtain
\begin{equation}
\begin{aligned}
&\frac{1}{T \eta}\mathbb{E}[L(\mathbf{Q}(T))-L(\mathbf{Q}(0))]+ \frac{1}{T}\sum_{t=0}^{T-1}\mathbb{E}[\Delta R(t)] \\
& \leq \frac{N}{2 \eta}+\frac{1}{T \eta}\sum_{t=0}^{T-1}\mathbb{E}[C_1(t)].
\end{aligned}
\end{equation}
Since $L(\mathbf{Q}(T)) \geq 0$ and $L(\mathbf{Q}(0))=0$, we have
\begin{equation}
\frac{1}{T}\sum_{t=0}^{T-1}\mathbb{E}[\Delta R(t)] \leq \frac{N}{2\eta }+ \frac{1}{T\eta }\sum_{t=0}^{T-1}\mathbb{E}\left[C_1(t)\right]. \label{eq:r_bar_ubound} 
\end{equation}

In Appendix~\ref{app:Handling A1}, we will show the following bound:

\begin{equation}
\begin{aligned}
&\frac{1}{T\eta }\sum_{t=0}^{T-1}\mathbb{E}[C_1(t)]\\
&\leq \frac{w_{\max}}{T} \left(2\sqrt{6mNT\log{T}}+(1+\frac{5\pi^2}{12})N\right).
\label{eq:C1_sum}
\end{aligned}
\end{equation}


Finally, plugging \eqref{eq:C1_sum} into \eqref{eq:r_bar_ubound} and combining it with \eqref{eq:regret_sum} yield \eqref{eq:regret result}. This completes the proof of Theorem \ref{thm:regret}.
\end{proof}

\subsection{Bounding $C_1(t)$} \label{app:Handling A1}
In this section, we want to show \eqref{eq:C1_sum}.


Consider a policy $\pi^{\prime}$, which, in each round $t$, chooses a super arm $S'(t)$ in the following manner:
\begin{equation}
S^{\prime}(t) \in \argmax_{\S \in \mathcal{S}(A(t))} \sum_{i \in S}(Q_i(t)+\eta w_i {\mu}_i(t)). 
\label{eq:S-prime}
\end{equation} 
Recall that in each round $t$, the LFG algorithm chooses a super arm $S(t)$ according to \eqref{best super arm}. Therefore, we have
\begin{equation}
\sum_{i \in S(t)}(Q_i(t)+\eta w_i{\bar{\mu}}_i(t)) \geq \sum_{i \in S^{\prime}(t)}(Q_i(t)+\eta w_i{\bar{\mu}}_i(t)). 
\label{gap_set}
\end{equation}

Next, we derive an upper bound on $C_1(t)$:
%
\begin{equation}
\begin{aligned}
C_1(t) &= \sum_{i=1}^{N}{(Q_i(t)+\eta w_i{\mu}_i)(d_i^*(t)-d_i(t))} \\
&= \sum_{i \in S^*(t)}{(Q_i(t)+\eta w_i{\mu}_i)} -\sum_{i \in S(t)}{(Q_i(t)+\eta w_i{\mu}_i)} \\
&\overset{(a)}{\leq} \sum_{i \in S^{\prime}(t)}{(Q_i(t)+\eta w_i{\mu}_i)} - \sum_{i \in S(t)}{(Q_i(t)+\eta w_i{\mu}_i)}\\
&\overset{(b)}{\leq} \sum_{i \in S^{\prime}(t)}{(Q_i(t)+\eta w_i{\mu}_i)} - \sum_{i \in S(t)}{(Q_i(t)+\eta w_i{\mu}_i)}\\
&\quad + \sum_{i \in S(t)}(Q_i(t)+\eta w_i{\bar{\mu}}_i(t)) \\
&\quad - \sum_{i \in S^{\prime}(t)}{(Q_i(t)+\eta w_i{\bar{\mu}}_i(t))}\\
&= \eta \left( \vphantom{\sum_{i \in S(t)}{w_i({\bar{\mu}}_i(t) -{\mu}_i)}} \right.
\underbrace{\sum_{i \in S(t)}{w_i({\bar{\mu}}_i(t) -{\mu}_i)}}_{C_2(t)} + \underbrace{\sum_{i \in S^{\prime}(t)}  w_i(\mu_i - {\bar{\mu}}_i(t))}_{C_3(t)}
\left. \vphantom{\sum_{i \in S^{\prime}(t)}  w_i(\mu_i - {\bar{\mu}}_i(t))} \right),
\end{aligned}
\label{eq:C1}
\end{equation}
where $(a)$ is from \eqref{eq:S-prime} and $(b)$ is from \eqref{gap_set}. Define $C_2(t) \triangleq  \sum_{i \in S(t)}{w_i({\bar{\mu}}_i(t) -{\mu}_i)}$ and $ C_3(t) \triangleq \sum_{i \in S^{\prime}(t)} w_i(\mu_i - {\bar{\mu}}_i(t))$. 
In Appendices~\ref{app:upper_bound_C2} and \ref{app:upper_bound_C3}, we will show the following two bounds, respectively:
\begin{align}
\sum_{t=0}^{T-1} \mathbb{E}[C_2(t)] 
& \leq w_{\max} \left(2\sqrt{6mNT\log{T}}+ (1+\frac{\pi^2}{4})N \right)  \label{eq:C2_sum} \\
\sum_{t=0}^{T-1}\mathbb{E}[C_3(t)]
& \leq \frac{\pi^2}{6} w_{\max} N. 
\label{eq:C3_sum}
\end{align}
%


Finally, summing \eqref{eq:C1} for all $t \in \{0,\dots, T-1\}$, dividing both sides of the resulting inequality by $T \eta$, and plugging \eqref{eq:C2_sum} and \eqref{eq:C3_sum} into it yield \eqref{eq:C1_sum}.

\emph{Remark}: The bound in \eqref{eq:C2_sum} consists of two terms: the first term is of the order $O(\sqrt{T \log{T}})$, which corresponds to the notion of regret in typical MAB problems and is attributed to the cost that needs to be paid in the learning/exploration process; the second term is a constant, which is from applying the Chernoff-Hoeffding bound (see, e.g., \cite{auer2002finite}) to a ``bad" event $\{\hat{\mu}_i(t-1)-{\mu}_i > \sqrt{\frac{3\log{t}}{2h_i(t-1)}}\}$. Similarly, the bound in \eqref{eq:C3_sum} is from applying the Chernoff-Hoeffding bound to another ``bad" event $\{\hat{\mu}_i(t-1)-{\mu}_i < -\sqrt{\frac{3\log{t}}{2h_i(t-1)}}\}$.


\subsection{Bounding $C_2(t)$} \label{app:upper_bound_C2}
In this section, we want to show \eqref{eq:C2_sum}.

Consider an arbitrary arm $i$ in $\N$ and an arbitrary round $t=0,1,\dots,T-1$. Let $t^i_a$ be the round in which arm $i$ is played for the $a$-th time. Recall that $h_i(t)$ is the number of times arm $i$ has been played by the end of round $t$. Clearly, we have
$d_i(t^i_a)=1$, $h_i(t^i_a) = a$, and $h_i(t^i_a -1)= a-1$ for all $a \in \{1,2,\dots,h_i(T-1)\}$. In addition, we also have
\begin{equation}
0 \leq t^i_1 < t^i_2 < \dots <t^i_{h_i(T-1)} < T.\label{eq:property}
\end{equation}

Define the following event: 
\begin{equation}
U_i(t) \triangleq \left\{ \bar{\mu}_i(t) < {\mu}_i\right\}.
\end{equation}
Let $E^c$ be the complement of an event $E$, and let $\mathbbm{1}_{\{\cdot\}}$ denote the indicator function. We bound the expectation of $C_2(t)$ as
\begin{equation}
\begin{aligned}
\mathbb{E}[C_2(t)] &= \mathbb{E}\left[\sum_{i=1}^N w_i(\bar{\mu}_i(t)-{\mu}_i)d_i(t)\right] \\ 
 &= \mathbb{E}\left[\sum_{i=1}^N w_i(\bar{\mu}_i(t)-{\mu}_i)d_i(t)\mathbbm{1}_{\{U_i(t)\}}\right]\\
 &\quad + \mathbb{E}\left[\sum_{i=1}^N w_i(\bar{\mu}_i(t)-{\mu}_i)d_i(t)\mathbbm{1}_{\{U_i^c(t)\}}\right] \\
&\overset{(a)}{\leq} \mathbb{E}\left[\sum_{i=1}^N w_i(\bar{\mu}_i(t)-{\mu}_i)d_i(t)\mathbbm{1}_{\{U_i^c(t)\}}\right] \\
&\overset{(b)}{\leq} w_{\max} \sum_{i=1}^N \mathbb{E}[ \underbrace{(\bar{\mu}_i(t)-{\mu}_i)d_i(t)\mathbbm{1}_{\{U_i^c(t)\}}}_{J_1(t)}],
\label{eq:C2_J1} 
\end{aligned}
\end{equation}
where $(a)$ is due to $\bar{\mu}_i(t) < {\mu}_i$ when event $U_i(t)$ happens and $(b)$ is due to $\bar{\mu}_i(t) \ge {\mu}_i$ when event $U^c_i(t)$ happens.
Define $J_1(t) \triangleq (\bar{\mu}_i(t)-{\mu}_i)d_i(t)\mathbbm{1}_{\{U_i^c(t)\}}$.
Also, define another event: 
\begin{equation}
F_i(t) \triangleq \left\{ \hat{\mu}_i(t-1)-{\mu}_i \leq \sqrt{\frac{3\log t}{2h_i(t-1)}} \right\}.    
\end{equation}
Then, summing $J_1(t)$ for all $t \in \{0,\dots, T-1\}$ gives
\begin{equation}
\begin{aligned}
&\sum_{t=0}^{T-1} J_1(t) \\
&\overset{(a)}{=} \sum_{a=1}^{h_i(T-1)} (\bar{\mu}_i(t^i_a)-{\mu}_i)\mathbbm{1}_{\{U_i^c(t^i_a)\}} \\
&\overset{(b)}{\leq} 1 + \sum_{a=2}^{h_i(T-1)}(\bar{\mu}_i(t^i_a)-{\mu}_i)\mathbbm{1}_{\{U_i^c(t^i_a)\}} \\
&= 1 + \sum_{a=2}^{h_i(T-1)}(\bar{\mu}_i(t^i_a)-{\mu}_i) \mathbbm{1}_{\{U_i^c(t^i_a)\}} (\mathbbm{1}_{\{F_i(t^i_a)\}} + \mathbbm{1}_{\{F^c_i(t^i_a)\}}) \\
&\overset{(c)}{\leq} 1 + \sum_{a=2}^{h_i(T-1)} (\underbrace{(\bar{\mu}_i(t^i_a)-{\mu}_i)\mathbbm{1}_{\{U_i^c(t^i_a) \cap F_i(t^i_a)\}}}_{J_2(t_a^i)} + \mathbbm{1}_{\{F_i^c(t^i_a)\}}), 
\label{eq:J1_sum}
\end{aligned}
\end{equation}
where $(a)$ is due to $d_i(t^i_a)=1$ for all $a \in \{1,2,\dots,h_i(T-1)\}$ and $d_i(t)=0$ for all other $t$, $(b)$ is due to $\bar{\mu}_i(t^i_1)-\mu_i \le 1$, and $(c)$ is due to $(\bar{\mu}_i(t^i_a)-\mu_i)\mathbbm{1}_{\{U_i^c(t^i_a)\}} \leq 1$.
We define $J_2(t_a^i) \triangleq (\bar{\mu}_i(t^i_a)-{\mu}_i)\mathbbm{1}_{\{U_i^c(t^i_a) \cap F_i(t^i_a)\}}$ and want to bound both $\sum_{a=2}^{h_i(T-1)} \mathbb{E}[J_2(t_a^i)]$ and $\sum_{a=2}^{h_i(T-1)} \mathbb{E}[\mathbbm{1}_{\{F_i^c(t^i_a)\}}]$. 

First, we want to bound $\sum_{a=2}^{h_i(T-1)} \mathbb{E}[J_2(t_a^i)]$. Consider $t^i_a$ for all $a \in \{2,\dots, h_i(T-1)\}$. Suppose event $F_i(t^i_a)$ happens. Then, we have
\begin{equation}
\hat{\mu}_i(t^i_a-1)-{\mu}_i \leq \sqrt{\frac{3\log{t^i_a}}{2h_i(t^i_a-1)}}.
\label{eq:hatmu_mu}
\end{equation}
From \eqref{UCB estimation}, we also have 
\begin{equation}
\bar{\mu}_i(t^i_a) \leq \hat{\mu}_i(t^i_a-1) + \sqrt{\frac{3\log{t^i_a}}{2h_i(t^i_a-1)}}.
\label{eq:barmu_hatmu}
\end{equation}
Combining \eqref{eq:hatmu_mu} and \eqref{eq:barmu_hatmu} gives
\begin{equation}
\bar{\mu}_i(t^i_a) - {\mu}_i \leq 2\sqrt{\frac{3\log{t^i_a}}{2h_i(t^i_a-1)}},
\end{equation}
which implies that for all $a \in \{2,\dots, h_i(T-1)\}$, we have
\begin{equation}
\begin{aligned}
J_2(t_a^i) &= (\bar{\mu}_i(t^i_a)-{\mu}_i)\mathbbm{1}_{\{U_i^c(t^i_a) \cap F_i(t^i_a)\}} \\
&\leq 2\sqrt{\frac{3\log{t^i_a}}{2h_i(t^i_a-1)}}.
\label{eq:J2}
\end{aligned}
\end{equation}

Then, summing $J_2(t^i_a)$ for all $a \in \{2,\dots, h_i(T-1)\}$ gives
\begin{equation}
\begin{aligned}
\sum_{a=2}^{h_i(T-1)} J_2(t^i_a) 
&\overset{(a)}{\leq} \sum_{a=2}^{h_i(T-1)}  2\sqrt{\frac{3\log t^i_a}{2h_i(t^i_a-1)}} \\
& \overset{(b)}{\leq} \sqrt{6\log{T}}\sum_{a=2}^{h_i(T-1)} \frac{1}{\sqrt{a-1}}  \\
& \overset{(c)}{\leq} \sqrt{6\log{T}}\left(1+\int_{1}^{h_i(T-1)}\frac{1}{\sqrt{x}} \, dx \right) \\
& \leq 2\sqrt{6h_i(T-1)\log{T}},
\end{aligned}
\end{equation}
where $(a)$ is from \eqref{eq:J2}, $(b)$ is due to $t^i_a \leq T$ (from \eqref{eq:property}) and $h_i(t^i_a-1) = a-1$ for all $a \in \{2,\dots, h_i(T-1)\}$, and $(c)$ is due to a basic relationship between the considered summation and integral. Therefore, we have
\begin{equation}
\sum_{a=2}^{h_i(T-1)} \mathbb{E}[J_2(t^i_a)] 
\leq 2\sqrt{6\log{T}} \mathbb{E}[\sqrt{h_i(T-1)}].
\label{eq:Expected_J2_sum}
\end{equation}


Next, we want to bound $\sum_{a=2}^{h_i(T-1)} \mathbb{E}[\mathbbm{1}_{\{F_i^c(t^i_a)\}}]$.
According to the definition of $t_a^i$, we have $h_i(t^i_a-1)=a-1$, and thus $\hat{\mu}_i(t_a^i-1)$ is the sample mean of $(a-1)$ \emph{i.i.d.} random variables $X_i(t_1^i),\cdots, X_i(t_{a-1}^i)$ with mean $\mu_i$. Further, we know $t^i_a$ must satisfy $a-1 \leq t^i_a \leq T-1$. Hence, 
\begin{equation}
\begin{aligned}
F_i^c(t^i_a)& =  \left\{\hat{\mu}_i(t_a^i-1)-{\mu}_i > \sqrt{\frac{3\log {t^i_a}}{2h_i(t^i_a-1)}}\right\}\\
&\subseteq \cup_{n=a-1}^{T-1} \left\{\hat{\mu}_i(n-1)-{\mu}_i > \sqrt{\frac{3\log {n}}{2(a-1)}}\right\}.
\end{aligned}
\end{equation}
By applying the union bound and the Chernoff-Hoeffding bound (see, e.g., \cite{auer2002finite}), we have 
\begin{equation}
\begin{aligned}
&\mathbb{E}[\mathbbm{1}_{\{F_i^c(t^i_a)\}}] 
= \mathbb{P}\left\{F_i^c(t^i_a)\right\} \\
&\leq \sum_{\tau=a-1}^{T-1} \mathbb{P}\left\{\hat{\mu}_i(\tau-1)-{\mu}_i > \sqrt{\frac{3\log {\tau}}{2(a-1)}}\right\} \\
&\leq  \sum_{\tau=a-1}^{T-1} \frac{1}{{\tau}^3} \leq \frac{1}{(a-1)^3}+\int_{a-1}^{\infty}\frac{1}{x^3} dx \leq \frac{3}{2(a-1)^2}.\notag
\end{aligned}
\end{equation}
Hence, we derive 
\begin{equation}
\begin{aligned}
\sum_{a=2}^{h_i(T-1)} \mathbb{E}[\mathbbm{1}_{\{F_i^c(t^i_a)\}}]
& \leq \sum_{a=2}^{h_i(T-1)}\frac{3}{2(a-1)^2}  \leq \sum_{a=1}^{\infty} \frac{3}{2a^2} = \frac{\pi^2}{4}.
\label{eq:Expected_Fc}
\end{aligned}
\end{equation}

Taking expectation of both sides of \eqref{eq:J1_sum} and plugging \eqref{eq:Expected_J2_sum} and \eqref{eq:Expected_Fc} into it yield
\begin{equation}
\sum_{t=0}^{T-1} \mathbb{E}[J_1(t)] \leq 2\sqrt{6\log{T}} \mathbb{E}[\sqrt{h_i(T-1)}] + 1+\frac{\pi^2}{4}.
\label{eq:J1_expectation} 
\end{equation}


Finally, summing \eqref{eq:C2_J1} for all $t \in \{0,\dots, T-1\}$ and plugging \eqref{eq:J1_expectation} into it yield \eqref{eq:C2_sum}:
\begin{equation}
\begin{aligned}
&\sum_{t=0}^{T-1} \mathbb{E}[C_2(t)]\\
&\leq w_{\max} \sum_{i=1}^{N} \left( 2\sqrt{6\log{T}} \mathbb{E}[\sqrt{h_i(T-1)}] + 1+\frac{\pi^2}{4} \right) \\
&\leq w_{\max} \left(2\sqrt{6mNT\log{T}}+ (1+\frac{\pi^2}{4})N \right), 
\end{aligned}
\end{equation}
where the last step follows from $\frac{1}{N}\sum_{i=1}^{N}\sqrt{h_i(T-1)} \leq \sqrt{\frac{1}{N}\sum_{i=1}^{N}h_i(T-1)}$ (due to Jensen's inequality) and $\sum_{i=1}^{N} h_i(T-1) \leq Tm$ (due to the fact that at most $m$ arms can be selected in each round).

\subsection{Bounding $C_3(t)$}\label{app:upper_bound_C3}

In this section, we want to show \eqref{eq:C3_sum}.

Consider an arbitrary arm $i$ in $\N$ and an arbitrary round $t=0,1,\dots,T-1$. Recall that $C_3(t)=\sum_{i \in S^{\prime}(t)} w_i(\mu_i - {\bar{\mu}}_i(t))$. Let $\mathbf{d}^{\prime}(t)=(d^{\prime}_1(t),\dots, d^{\prime}_N(t))$ be the action vector corresponding to $S^{\prime}(t)$. Also, recall that $U_i(t) = \left\{ \bar{\mu}_i(t) < {\mu}_i\right\}$. Similar to the derivation for $C_2(t)$ in \eqref{eq:C2_J1}, we bound the expectation of $C_3(t)$ as
\begin{equation}
\begin{aligned}
\mathbb{E}[C_3(t)] 
&= \mathbb{E}\left[\sum_{i=1}^{N}w_i ({\mu}_i-\bar{\mu}_i(t))d_i^{\prime}(t) \right]\\
&= \mathbb{E}\left[\sum_{i=1}^{N}w_i ({\mu}_i-\bar{\mu}_i(t))d_i^{\prime}(t)\mathbbm{1}_{\{U_i(t)\}} \right]\\
&\quad + \mathbb{E}\left[\sum_{i=1}^{N}w_i ({\mu}_i-\bar{\mu}_i(t))d_i^{\prime}(t)\mathbbm{1}_{\{U_i^c(t)\}} \right]\\
&\overset{(a)}{\leq} \mathbb{E}\left[\sum_{i=1}^{N}w_i ({\mu}_i-\bar{\mu}_i(t))d_i^{\prime}(t)\mathbbm{1}_{\{U_i(t)\}} \right]\\
&\overset{(b)}{\leq} w_{\max} \sum_{i=1}^{N} \mathbb{E} [ \underbrace{({\mu}_i-\bar{\mu}_i(t))d_i^{\prime}(t)\mathbbm{1}_{\{U_i(t)\}}}_{K_1(t)} ],
\label{eq:C3_K1}
\end{aligned}
\end{equation}
where $(a)$ is due to $\bar{\mu}_i(t) \geq {\mu}_i$ when event $U^c_i(t)$ happens, and $(b)$ is due to $\bar{\mu}_i(t) < {\mu}_i$ when event $U_i(t)$ happens. 
We define $K_1(t) \triangleq ({\mu}_i-\bar{\mu}_i(t))d_i^{\prime}(t)\mathbbm{1}_{\{U_i(t)\}}$ and consider two cases for $\mathbb{E}[K_1(t)]$: i) $t \leq t_1^i$ and ii) $t > t_1^i$. 

In Case i), event $U_i(t)$ must not happen, i.e.,  $\bar{\mu}_i(t) \geq \mu_i$ must hold. This is because $\bar{\mu}_i(t) = 1$ (due to $h_i(t-1)=0$ for $t \leq t_1^i$) and $\mu_i \in [0,1]$. Hence, for all $t \leq t_1^i$ we have 
\begin{equation}
\mathbb{E}[K_1(t)]=0.
\label{eq:K1_Casei}
\end{equation}

In Case ii), suppose event $U_i(t)$ happens. Then we have $\bar{\mu}_i(t) < {\mu}_i \le 1$ and $1\leq h_i(t-1)\leq t$. This, along with \eqref{UCB estimation}, implies $\bar{\mu}_i(t) = \hat{\mu}_i(t-1) + \sqrt{\frac{3\log{t}}{2h_i(t-1)}}$, which further implies
\begin{equation}
U_i(t) = \left\{\hat{\mu}_i(t-1)-{\mu}_i < -\sqrt{\frac{3\log{t}}{2h_i(t-1)}}\right\}. \label{eq:event_Ui}
\end{equation}
This leads to the following bound on $\mathbb{E}[K_1(t)]$ for all $t > t_1^i$:
\begin{equation}
\begin{aligned}
\mathbb{E}[K_1(t)]&= \mathbb{E}[({\mu}_i-\bar{\mu}_i(t)) d_i^{\prime}(t) \mathbbm{1}_{\{U_i(t)\}}] \\
& \leq \mathbb{E}[\mathbbm{1}_{\{U_i(t)\}}] = \mathbb{P}\{U_i(t)\} \\
&= \mathbb{P} \left\{\hat{\mu}_i(t-1)-{\mu}_i < -\sqrt{\frac{3\log{t}}{2h_i(t-1)}}\right\} 
\end{aligned}
\end{equation}
where the inequality is due to ${\mu}_i - \bar{\mu}_i(t) \leq 1$ and $d_i^{\prime}(t) \leq 1$. Note that $\hat{\mu}_i(t-1)$ in Eq. \eqref{eq:event_Ui} is the sample mean of $h_i(t-1)$ \emph{i.i.d.} random variables (donoted as $X(1), X(2), \cdots$) in $[0,1]$ with mean $\mu_i$ and that $h_i(t-1)$ is random but has finite possible values $\{1, \cdots, t\}$. Hence, for each possible value of $h_i(t-1)$, the Chernodd-Hoeffding Bound could be applied. Therefore, we have 
\begin{equation}
\begin{aligned}
&\mathbb{E}[K_1(t)]\leq \mathbb{P}\{U_i(t)\} \\
= &\sum_{h=1}^t \mathbb{P} \left\{\left\{h_i(t\!-\!1)\!=\!h\right\}\cap \left\{\hat{\mu}_i(t\!-\!1)\!-\!{\mu}_i \!<\! \!-\!\sqrt{\frac{3\log{t}}{2h_i(t\!-\!1)}}\right\}\right\} \\
\leq & \sum_{h=1}^t \mathbb{P} \left\{\frac{1}{h}\sum_{n=1}^h X(n)-{\mu}_i < -\sqrt{\frac{3\log{t}}{2h}}\right\} \\\label{eq:K1_Caseii}
\leq & \sum_{h=1}^t \frac{1}{t^3} =\frac{1}{t^2},
\end{aligned}
\end{equation}
where the last inequality is from the Chernoff-Hoeffding Bound (see, e.g., \cite{auer2002finite}).

Summing $\mathbb{E}[K_1(t)]$ for all $t \in \{0,\dots, T-1\}$ and applying \eqref{eq:K1_Casei} and \eqref{eq:K1_Caseii} yield
\begin{equation}
\begin{aligned}
\sum_{t=0}^{T-1} \mathbb{E}[K_1(t)] 
\leq \sum_{t=t^i_1+1}^{T-1} \frac{1}{t^2} \leq \sum_{t=1}^{\infty} \frac{1}{t^2} = \frac{\pi^2}{6}.
\label{eq:K1_expectation}
\end{aligned}
\end{equation}


Finally, Summing \eqref{eq:C3_K1} for all $t \in \{0,\dots, T-1\}$ and plugging \eqref{eq:K1_expectation} into it yield \eqref{eq:C3_sum}.

\end{document}